\documentclass[letterpaper]{article} 
\usepackage{aaai2027}  
\usepackage[hyphens]{url}  
\usepackage{graphicx} 
\usepackage{natbib}  
\usepackage{caption} 
\usepackage{algorithm}
\usepackage{algorithmic}
\usepackage{microtype}
\usepackage{newfloat}
\usepackage{listings}
\DeclareCaptionStyle{ruled}{labelfont=normalfont,labelsep=colon,strut=off} 
\floatstyle{ruled}
\newfloat{listing}{tb}{lst}{}
\floatname{listing}{Listing}

\usepackage{booktabs}
\usepackage{amsmath}
\usepackage{amssymb}
\usepackage{amsthm}

\newtheorem{theorem}{Theorem}
\newtheorem{proposition}{Proposition}
\newtheorem{lemma}{Lemma}
\newtheorem{corollary}{Corollary}
\newtheorem{definition}{Definition}
\newtheorem{remark}[theorem]{Remark}

\newcommand{\appProofs}{Appendix~A}       
\newcommand{\appStructure}{Appendix~B}     
\newcommand{\appGeometry}{Appendix~C}    
\newcommand{\appReg}{Appendix~D}          
\newcommand{\appApprox}{Appendix~E}       
\newcommand{\appProtocol}{Appendix~F}      
\newcommand{\appStats}{Appendix~G}       
\newcommand{\appReach}{Appendix~H}      
\newcommand{\appUnified}{Appendix~I}       
\newcommand{\appDepth}{Appendix~J}        
\nocopyright 

\title{FishBack: Pullback Fisher Geometry for Optimal Activation Steering in Transformers}
\author{
	Sihan Wang\textsuperscript{1},
	Jiayi Zhao\textsuperscript{1},
	Qingyan Cao \textsuperscript{1},
	Hongbo Yao\textsuperscript{1}\thanks{Corresponding author.},
	Lin Shu\textsuperscript{1}\thanks{Corresponding author.}
}
\affiliations{
	\textsuperscript{1}School of Future Technology, South China University of Technology, Guangzhou, China\\
	\{cswangsihan, 202564871285, 202310193229\}@mail.scut.edu.cn, shul@scut.edu.cn
}

\begin{document}

\maketitle

\begin{abstract}
Activation steering has emerged as a lightweight approach for modifying language model behavior without parameter updates, yet existing methods remain brittle: unstable across layers and prone to disturbing behavior unrelated to the target concept. We trace these failures to a hidden assumption shared by widely-used methods such as CAA, ActAdd, and ITI: that the intermediate activation space is Euclidean. We show this assumption is fundamentally flawed. The metric that actually governs how a hidden-state perturbation changes the output is the Fisher information metric of the softmax layer, pulled back to the intermediate layer through the Jacobian of the intervening layers. From it we derive a closed-form steering direction, applied to a hidden state at an intermediate layer, that reaches a target concept change with the least non-target distortion. The framework is sharpest in the early and middle intermediate layers, where the metric is strongly non-Euclidean and geometric correction matters most. We evaluate it on three verb-morphology concepts: third-person-singular, progressive, and past-tense inflection, following standard counterfactual-concept evaluation. On GPT-2 Small, this non-Euclidean geometry is borne out empirically, and our method lowers off-target KL divergence by median factors of 1.4--6.5x against individual steering baselines. On Llama-3-8B and Qwen3-8B, it lowers off-target KL by median factors of 1.8--3.6x against individual baselines at the early and middle layers. These results show that geometric correction retains its advantage on larger models with more complex internal structure.
\end{abstract}

\begin{figure*}[t]
	\centering
	\includegraphics[width=\textwidth]{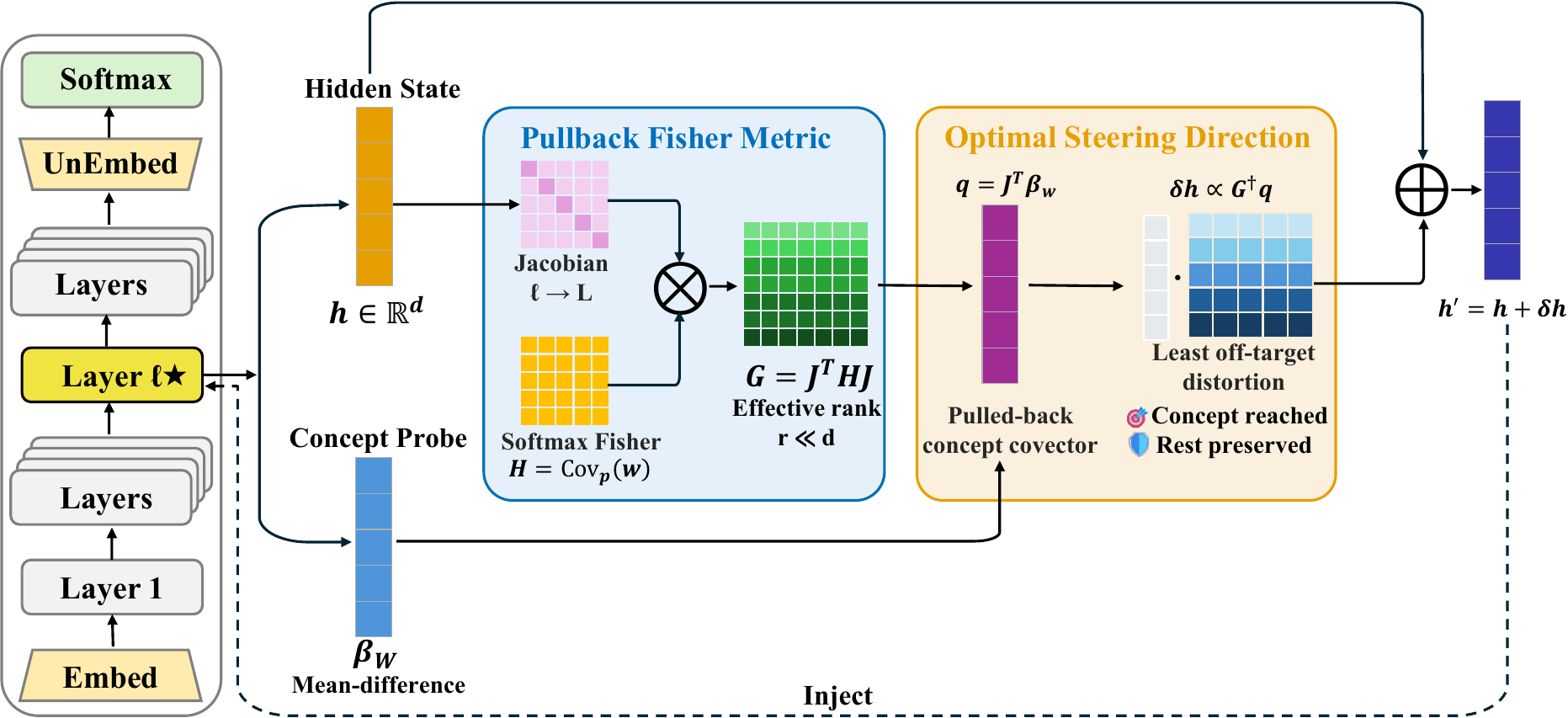}
	\caption{\textbf{Overview of FishBack.} At an intermediate layer, we read the hidden state and the concept probe, pull the softmax Fisher information back through the network's Jacobian to obtain the pullback Fisher metric, and solve in closed form for the steering direction. The steered activation is then injected back into the layer.}
	\label{fig:overview}
\end{figure*}
\section{Introduction}

Shaping a language model's behavior after training---keeping it truthful, within a safety policy, or free of toxic and biased text---is a standing need, and reweighting the model through fine-tuning or human feedback \citep{ouyang2022training} is costly and specific to each objective. Activation steering offers a lighter alternative: because a model's concepts tend to lie along linear directions in its hidden states, displacing an activation along the right direction at inference time turns a behavior up or down \citep{subramani2022extracting,turner2023activation,zou2023representation}. A large toolkit now builds such directions---from contrastive mean differences \citep{rimsky2024steering}, probes on individual heads \citep{li2023inference}, and extracted task or function vectors \citep{hendel2023incontext,todd2024function}, to fitted affine and rank-one directions \citep{singh2024representation,belrose2023leace}---and it has served truthfulness \citep{li2023inference}, sentiment \citep{turner2023activation}, and refusal \citep{arditi2024refusal}. Yet these interventions are unstable: a vector that works at one layer fails at another, its strength must be tuned within narrow margins \citep{weij2024extending,vu2025angular}, its effect leaks into unrelated behavior, and audits report steerability that varies by input, sometimes reverses sign \citep{tan2024analysing,braun2025understanding}, and barely improves on simple baselines \citep{wu2025axbench}. Recurring across models and tasks, this points to a common cause.

That cause is the form of the intervention. The most widely used methods---CAA, ActAdd, and ITI among them---apply the same update $h \leftarrow h + \alpha v$, which treats the activation space as flat, as if a step cost the same in every direction, and steers along its straight lines. The assumption is seldom stated, yet it decides which direction is steepest and whether shifting the target concept also disturbs the rest. It is known to be wrong at the network's final layer: there the softmax head maps the last hidden state to a distribution over tokens, whose local geometry is the Fisher--Rao metric of information theory \citep{rao1945information,amari1998natural,amari2016information}, and steering in its dual coordinates best preserves unrelated semantics \citep{park2026information}. But steering is almost never applied at that final layer. It acts on the intermediate layers of the residual stream, far upstream of the softmax, where the operative geometry has gone unexamined. Recent work fits this geometry from activation samples---a manifold or a kernel-PCA surface \citep{wurgaft2026manifold,raval2026curveball}---confirming that straight-line steering is wrong, but stops short of deriving it from the model or guaranteeing an optimal direction. What is missing is a metric for the intermediate layers that follows from the model itself and names the best direction to move.

We show that this metric need not be fitted at all; it is fixed by the model's own computation. What a displacement costs is not its length but how much it moves the model's predicted distribution, and that cost is available in closed form: the Fisher information of the softmax head, carried back to an intermediate layer along the network's Jacobian \citep{arvanitidis2018latent,arvanitidis2022pulling}, gives a layer-specific metric on the residual stream, built from quantities the model already computes and equal to the generalized Gauss--Newton matrix of second-order optimization \citep{martens2020new}. Measured on GPT-2 Small, it is far from flat. Because the metric is explicit, so is the optimal direction: a single regularized solve returns the displacement that induces a target concept change at the least off-target distortion, and the solve is reapplied along the trajectory as the operating point moves. At the scale of an 8B model this computation becomes infeasible---assembling the Jacobian alone would take thousands of Jacobian-vector products at every step---so we develop an approximation that recovers the same direction without forming it \citep{halko2011finding}, bringing the exact target within a few hundred partial forward passes per case on Llama-3-8B and Qwen3-8B.

Our contributions are as follows.
\begin{itemize}
	\item \textbf{The correct geometry for intermediate-layer steering.} We show that the metric governing an intervention at an intermediate layer is the softmax Fisher information pulled back through the network's Jacobian, obtained in closed form from the model with no fitting to data; the widely used additive methods correspond to steering under a flat metric, with directions that differ from ours in the covector as well (\S4).
	\item \textbf{A closed-form optimal steering direction with state-of-the-art off-target preservation.} From this metric we prove a closed-form expression for the displacement that produces a target concept change at the least off-target KL divergence, applied iteratively along the steering trajectory. On GPT-2 Small, Llama-3-8B, and Qwen3-8B, it attains state-of-the-art off-target preservation among activation-steering methods at the early and middle layers, and the highest target reachability on Qwen3-8B (\S4, \S5).
	\item \textbf{A faithful approximation that scales to 8B models.} Because forming the metric exactly is computationally out of reach at 8B scale, we give an approximation that recovers the same optimal direction---evaluating its action and low-rank spectrum without forming it, and holding the Jacobian fixed across steps---within a few hundred partial forward passes per case on Llama-3-8B and Qwen3-8B (\S4 \S5).
\end{itemize}

\section{Related Work}
\label{sec:related}

\textbf{Activation steering} exploits linear concept directions in hidden states \citep{subramani2022extracting,zou2023representation}, differing mainly in how the direction is obtained---contrastive means \citep{turner2023activation,rimsky2024steering}, head probes \citep{li2023inference}, task and function vectors \citep{hendel2023incontext,todd2024function}, affine or rank-one fits \citep{singh2024representation,belrose2023leace}---while related work controls generation without modifying activations additively \citep{dathathri2020plug,im2025unified}. The additive methods are widely reported to be fragile across layers, coefficients, and inputs \citep{weij2024extending,vu2025angular,tan2024analysing,braun2025understanding,wu2025axbench}.

\textbf{Geometry-aware steering} questions the flat view: \citet{park2026information} steer in the softmax's dual coordinates, but at the output layer only, while \citet{wurgaft2026manifold} and \citet{raval2026curveball} fit a manifold or kernel-PCA surface to activations---confirming curvature but obtaining the geometry from data rather than the model, with no optimality guarantee.

\textbf{Information geometry} grounds our approach: our metric follows from the model in closed form, connecting to the Fisher--Rao metric \citep{rao1945information,amari2016information}, natural-gradient optimization \citep{amari1998natural}, generalized Gauss--Newton curvature \citep{martens2020new}, and pullback metrics on generative latent spaces \citep{arvanitidis2018latent,arvanitidis2022pulling}, none of which has been used to steer intermediate representations.
\section{Preliminaries}
\label{sec:prelim}

We write scalars in lightface ($c$), vectors in lowercase bold ($\mathbf{v}$), and matrices in uppercase ($M$); $\mathbb{E}$ denotes expectation and $D_{\mathrm{KL}}$ the Kullback--Leibler divergence.

\paragraph{Network and output distribution.}
We intervene on the residual stream at an intermediate layer $\ell$ and read the effect in the model's next-token distribution. Let $\mathbf{h} \in \mathbb{R}^{d}$ be the activation at layer $\ell$ and $\boldsymbol{\lambda} \in \mathbb{R}^{d}$ the final-layer representation, related by the smooth map $\boldsymbol{\lambda} = f(\mathbf{h})$ that the remaining layers compute, with Jacobian $J = \partial \boldsymbol{\lambda} / \partial \mathbf{h}$. The softmax head produces a distribution over the vocabulary $\mathcal{V}$,
\begin{equation}
	\label{eq:softmax}
	p_{\boldsymbol{\lambda}} = \mathrm{softmax}(W_U \boldsymbol{\lambda}), \qquad W_U \in \mathbb{R}^{|\mathcal{V}| \times d},
\end{equation}
through the unembedding matrix $W_U$, whose rows we write $\mathbf{w}_i$.

\paragraph{The output metric.}
Change in the prediction is measured by KL divergence, not Euclidean length, and to second order this divergence is the quadratic form of the Fisher information $H$:
\begin{equation}
	\label{eq:kl_fisher}
	\begin{aligned}
		D_{\mathrm{KL}}\!\left(p_{\boldsymbol{\lambda}} \,\big\|\, p_{\boldsymbol{\lambda} + \Delta\boldsymbol{\lambda}}\right)
		&\approx \tfrac{1}{2}\, \Delta\boldsymbol{\lambda}^{\top} H \, \Delta\boldsymbol{\lambda}, \\
		H &= \mathbb{E}_{p_{\boldsymbol{\lambda}}}\!\big[\mathbf{w}\mathbf{w}^{\top}\big] - \bar{\mathbf{w}}\,\bar{\mathbf{w}}^{\top},
	\end{aligned}
\end{equation}
the covariance of the unembedding rows under the current prediction, where $\bar{\mathbf{w}} = \mathbb{E}_{p_{\boldsymbol{\lambda}}}[\mathbf{w}]$. This holds when the second-order term dominates, the regime in which we steer; we check its accuracy on real models in Section~\ref{sec:experiments}. Because $p_{\boldsymbol{\lambda}}$ concentrates on a few plausible tokens, $H$ is far from isotropic, with most of its weight in a low-dimensional subspace---a fact that drives everything below.

\paragraph{Concepts and off-target distortion.}
Following \citet{park2026information}, a binary concept $W$ is specified by counterfactual token pairs whose members differ only in the target attribute (e.g., \emph{run} vs.\ \emph{runs}), and is read by a mean-difference probe $\boldsymbol{\beta}_W$ on the final representation. Steering should move the concept while leaving unrelated predictions intact. To make this precise, partition the vocabulary into the concept tokens $\mathcal{Y}_W$ and the remainder $\mathcal{Y}_W^{c}$; the effect on $\mathcal{Y}_W^{c}$ is the \emph{off-target distortion},
\begin{equation}
	\label{eq:offtarget}
	\begin{split}
		D_{\mathrm{KL}}^{\mathrm{off}} &= D_{\mathrm{KL}}\!\left(p^{\mathrm{off}}_{\boldsymbol{\lambda}} \,\big\|\, p^{\mathrm{off}}_{\boldsymbol{\lambda}'}\right), \\
		p^{\mathrm{off}}(y) &= \frac{p(y)}{\sum_{y' \in \mathcal{Y}_W^{c}} p(y')} \quad \text{for } y \in \mathcal{Y}_W^{c},
	\end{split}
\end{equation}
the divergence between the distributions renormalized over $\mathcal{Y}_W^{c}$ before ($\boldsymbol{\lambda}$) and after ($\boldsymbol{\lambda}'$) steering. The evaluation protocol and concept-probability target are detailed in Section~\ref{sec:experiments}.

\paragraph{From concept direction to steering direction.}
Pulled back to the intervention layer, the probe becomes the covector
\begin{equation}
	\label{eq:pullback_covector}
	\mathbf{q} = J^{\top} \boldsymbol{\beta}_W,
\end{equation}
the direction along which the concept score climbs fastest at $\mathbf{h}$. Steering adds a \emph{vector} to $\mathbf{h}$, so turning the covector $\mathbf{q}$ into one requires a metric $M$ to raise its index, giving $\delta\mathbf{h} \propto M^{-1} \mathbf{q}$. Which $M$ to use is the entire question. Additive steering takes $M = I$, moving along $\mathbf{q}$ as if it were already a vector---the flat choice we make precise, and correct, in Section~\ref{sec:pullback}.

\section{Method}
\label{sec:pullback}

\subsection{The Pullback Metric}
\label{subsec:pullback_metric}

An intervention is applied to the hidden state $\mathbf{h}$ at layer $\ell$, yet what we care about is its effect several layers later, on the prediction. A perturbation $\delta\mathbf{h}$ reaches the final representation as $\Delta\boldsymbol{\lambda} = J\,\delta\mathbf{h}$ to first order, so substituting into the output-space cost \eqref{eq:kl_fisher} expresses that effect directly in terms of $\delta\mathbf{h}$.

\begin{theorem}[Second-order steering cost]
	\label{thm:pullback}
	To second order, the KL divergence between the prediction before and after an intervention $\delta\mathbf{h}$ at layer $\ell$ is
	\begin{equation}
		\label{eq:pullback}
		D_{\mathrm{KL}} \approx \tfrac{1}{2}\, \delta\mathbf{h}^{\top} G\, \delta\mathbf{h},
		\qquad
		G = J^{\top} H J,
	\end{equation}
	where $J$ is the Jacobian of the layers from $\ell$ to the output and $H$ is the softmax Fisher information. The matrix $G$ is symmetric positive semi-definite.
\end{theorem}

We call $G$ the \emph{pullback Fisher metric}: it carries the Fisher geometry of the output back to layer $\ell$ through the Jacobian, and \eqref{eq:pullback} shows that it is the quadratic form measuring how far the prediction moves under an intervention. The steering cost of a direction is therefore not its Euclidean length but its norm under $G$, which is highly anisotropic (Section~\ref{sec:experiments}). The proof substitutes $\Delta\boldsymbol{\lambda} = J\,\delta\mathbf{h}$ into \eqref{eq:kl_fisher}, with positive semi-definiteness inherited from $H$; details are in \appProofs.

The same matrix arises in second-order optimization.

\begin{proposition}[Gauss--Newton form]
	\label{prop:ggn}
	$G = J^{\top} H J$ is the generalized Gauss--Newton matrix of the model's negative log-likelihood with respect to the activation $\mathbf{h}$ at layer $\ell$.
\end{proposition}

The curvature that natural-gradient methods use to precondition parameter updates is thus, restricted to one layer's activations, the metric that governs steering there \citep{martens2020new}.
\subsection{Optimal Steering Direction}
\label{subsec:optimal}
The metric of Section~\ref{subsec:pullback_metric} reduces the choice of a steering direction to a constrained optimization. An effective intervention must induce a prescribed change in the target concept while minimizing the accompanying movement of the prediction, which Theorem~\ref{thm:pullback} quantifies as $\tfrac{1}{2}\,\delta\mathbf{h}^{\top} G\,\delta\mathbf{h}$. The concept is captured by the pulled-back covector $\mathbf{q} = J^{\top}\boldsymbol{\beta}_W$ of \eqref{eq:pullback_covector}, and a first-order concept change of magnitude $\rho$ corresponds to the linear constraint $\mathbf{q}^{\top}\delta\mathbf{h} = \rho$. The resulting quadratic program admits a closed-form solution.

\begin{theorem}[Optimal steering direction]
	\label{thm:optimal_steering}
	Suppose $\mathbf{q} \in \operatorname{Range}(G)$. Among all interventions achieving a concept change $\mathbf{q}^{\top}\delta\mathbf{h} = \rho$, the minimum-norm minimizer of the output movement $\tfrac{1}{2}\,\delta\mathbf{h}^{\top} G\,\delta\mathbf{h}$ is
	\begin{equation}
		\label{eq:optimal}
		\delta\mathbf{h}^{\star} = \frac{\rho}{\mathbf{q}^{\top} G^{\dagger} \mathbf{q}}\, G^{\dagger}\mathbf{q},
	\end{equation}
	where $G^{\dagger}$ is the Moore--Penrose pseudoinverse of $G$.
\end{theorem}

A Lagrange multiplier turns the constraint into the stationarity condition $G\,\delta\mathbf{h} = \nu\mathbf{q}$, which \eqref{eq:optimal} solves while matching the constraint. When $G$ is rank-deficient the minimizers form an affine family differing by directions in $\ker(G)$, all of equal second-order cost; \eqref{eq:optimal} selects the minimum-norm element, and the regularization of Section~\ref{subsec:computation} makes $\mathbf{q} \in \operatorname{Range}(G)$ hold in practice by replacing $G$ with $G + \alpha I$. The full derivation is in \appProofs. The direction is that of $G^{\dagger}\mathbf{q}$, not of $\mathbf{q}$: the metric rotates the raw concept covector toward the cheap directions of the output geometry, and the scalar prefactor sets its length to meet the target $\rho$.

Existing additive methods correspond to this rule under a Euclidean metric. Setting $G = I$ makes \eqref{eq:optimal} collapse to $\delta\mathbf{h}^{\star} = (\rho / \lVert\mathbf{q}\rVert^{2})\,\mathbf{q}$, a step along $\mathbf{q}$ itself; methods that add a fixed direction take the same flat metric, though their direction is an empirical estimate rather than the pulled-back covector.
\subsection{Preservation of Non-Target Behavior}
\label{subsec:nontarget}

Theorem~\ref{thm:optimal_steering} minimizes the total movement of the prediction, whereas the quantity of practical interest is the movement of the non-target distribution, since the on-target change is intended. These two objectives coincide under a separability condition on the concept: the predicted distribution factorizes into a concept factor carried by $\mathcal{Y}_W$ and a content factor shared with $\mathcal{Y}_W^{c}$, which holds exactly when all counterfactual pairs share the same within-pair log-odds. \appProofs{} states the condition formally and proves this criterion.

\begin{theorem}[Non-target optimality]
	\label{thm:nontarget}
Suppose the concept $W$ is separable in the sense above. Then, among all interventions achieving a concept change $\mathbf{q}^{\top}\delta\mathbf{h} = \rho$, the optimal direction $\delta\mathbf{h}^{\star}$ of \eqref{eq:optimal} also minimizes the off-target divergence $D_{\mathrm{KL}}^{\mathrm{off}}$ of \eqref{eq:offtarget} to second order.
\end{theorem}

Separability decomposes the output movement into a target and a non-target term, and the constraint $\mathbf{q}^{\top}\delta\mathbf{h} = \rho$ fixes the former; minimizing the total is therefore equivalent to minimizing the latter alone. The proof is given in \appProofs. The direction that is economical in aggregate is therefore economical precisely where it matters.
\subsection{Structure and Computation}
\label{subsec:computation}

The pullback metric inherits structure from the layers it is pulled through. Writing the Jacobian as the product of per-layer Jacobians expresses $G^{(\ell)}$ recursively in terms of $G^{(\ell+1)}$ (\appStructure), so the metric at an early layer accumulates the transformations of every layer above it. One effect of this accumulation is a growth in the anisotropy of $G$ with depth, for which the per-layer factors provide an upper bound on the condition number (\appStructure). This bound constrains how large the anisotropy can become but does not by itself determine where steering is most effective, a question we return to empirically in Section~\ref{sec:experiments}.

Two properties of $G$ shape its use in practice. First, the metric is rank-deficient and ill-conditioned, so the pseudoinverse in \eqref{eq:optimal} must be regularized; we use a scale-equivariant Tikhonov term whose strength adapts to the spectrum of $G$, which leaves the solution invariant to a rescaling of the activations and introduces a controlled bias (\appReg). Second, the direction of \eqref{eq:optimal} is a first-order construction, yet it remains accurate beyond the regime in which the second-order cost of Theorem~\ref{thm:pullback} is itself accurate: the deviation of the direction is of order $\rho^{2}$ (\appReg). This stability licenses an iterative procedure, in which the metric and the covector are recomputed at the updated activation and a further step is taken, so that a large concept change is reached through a sequence of steps each solved in closed form.

Forming $G$ explicitly is infeasible at the scale of an $8$B model, where assembling the Jacobian alone would require thousands of Jacobian-vector products at every step. The direction can nonetheless be computed matrix-free, as described in Section~\ref{sec:experiments} and \appApprox.

\section{Experiments}
\label{sec:experiments}

The experiments address three questions in the order the theory raises them: whether the pullback metric of a trained Transformer departs from the Euclidean metric at the layers where steering is applied (\S\ref{subsec:exp_geometry}); whether the closed-form direction of Theorem~\ref{thm:optimal_steering} outperforms existing steering methods at matched concept change (\S\ref{subsec:exp_offtarget}); and whether the observed gain is attributable to the metric itself (\S\ref{subsec:exp_ablation}). Section~\ref{subsec:exp_analysis} examines how the gain depends on depth.

\subsection{Setup}
\label{subsec:exp_setup}

\paragraph{Models and layers.}
We evaluate on GPT-2 Small \citep{radford2019language} ($d=768$), where every quantity in \eqref{eq:optimal} is computed exactly, and on Llama-3-8B \citep{grattafiori2024llama} and Qwen3-8B \citep{yang2025qwen3} ($d=4096$), where forming the metric exactly is computationally out of reach and the approximation summarized below is used. Interventions are applied at layers spanning early to late depth: $\{3,6,9,11\}$ for GPT-2 Small, $\{8,16,24\}$ for Llama-3-8B, and $\{6,18,27\}$ for Qwen3-8B. The main comparisons cover all four GPT-2 Small layers and the early and middle layers of the two 8B models; the deepest 8B layers are analyzed separately in Section~\ref{subsec:exp_analysis}.
\begin{table*}[t]
	\centering
	\small
	\begin{tabular}{l ccc ccc ccc}
		\toprule
		& \multicolumn{3}{c}{GPT-2 Small (L3--L11)} & \multicolumn{3}{c}{Llama-3-8B (L8, L16)} & \multicolumn{3}{c}{Qwen3-8B (L6, L18)} \\
		\cmidrule(lr){2-4} \cmidrule(lr){5-7} \cmidrule(lr){8-10}
		Method & Ratio [IQR] & Win\% & $n$ & Ratio [IQR] & Win\% & $n$ & Ratio [IQR] & Win\% & $n$ \\
		\midrule
		CAA        & 1.42 [0.81, 3.03] & 68.7 & 332 & 2.22 [0.88, 5.66] & 69.0 & 142 & 1.77 [0.94, 5.19] & 71.3 & 223 \\
		ActAdd     & 6.52 [2.32, 30.89] & 93.9 & 198 & 3.50 [1.63, 10.98] & 84.5 & 84 & 2.93 [1.38, 8.47] & 85.0 & 206 \\
		ITI        & 2.01 [1.11, 5.59] & 80.1 & 337 & 3.08 [1.05, 7.91] & 77.9 & 113 & 3.61 [1.37, 9.44] & 87.6 & 193 \\
		ReFT       & 2.67 [1.36, 8.52] & 87.2 & 329 & 2.20 [0.89, 5.60] & 71.1 & 159 & 1.93 [0.94, 5.87] & 70.3 & 192 \\
		RepSurgery & 1.97 [1.16, 4.99] & 81.0 & 327 & 2.75 [1.06, 7.07] & 76.6 & 141 & 2.11 [1.00, 5.72] & 74.5 & 235 \\
		\midrule
		Euclidean (ablation) & 1.30 [1.03, 1.82] & 80.1 & 366 & 1.79 [0.90, 5.13] & 70.3 & 236 & 2.45 [1.38, 4.84] & 86.7 & 165 \\
		\midrule
		Best-per-case (oracle) & 1.25 [0.75, 2.63] & 63.5 & 359 & 1.64 [0.75, 4.35] & 64.3 & 221 & 1.49 [0.89, 4.29] & 66.9 & 399 \\
		\bottomrule
	\end{tabular}
	\caption{FishBack lowers off-target KL divergence at matched concept change against every steering baseline, the Euclidean ablation, and a per-case oracle. Entries are the median [IQR] of the per-case ratio $\mathrm{KL}_{\mathrm{method}}/\mathrm{KL}_{\mathrm{FishBack}}$, pooled over concepts and targets; values above $1$ favor FishBack. Pooled entries are descriptive; per-cell tests appear in \appStats.}
	\label{tab:main}
\end{table*}
\paragraph{Concepts and evaluation.}
We adopt the counterfactual protocol of \citet{park2026information}. A concept is specified by token pairs differing only in the target attribute; we use three verb-morphology concepts---third-person-singular (\emph{run}$\to$\emph{runs}), progressive (\emph{run}$\to$\emph{running}), and past tense (\emph{run}$\to$\emph{ran})---each instantiated in natural-language contexts drawn from the C4 validation set \citep{raffel2020exploring} and retained by the acceptance criterion of that protocol: $50$ contexts per layer on GPT-2 Small and Llama-3-8B and $100$ on Qwen3-8B (\appProtocol). The concept probability
\begin{equation}
	\label{eq:concept_prob}
	P^{W} = \frac{\sum_{i} p(y_i^{1})}{\sum_{i} p(y_i^{0}) + \sum_{i} p(y_i^{1})}
\end{equation}
is a ratio of token masses that bisection can drive to a prescribed value, and the vocabulary partitions cleanly into concept and non-concept tokens, so the off-target divergence \eqref{eq:offtarget} at matched concept change is well defined. Under the separability condition of Section~\ref{subsec:nontarget}, $P^{W}$ is a strictly monotone function of the concept score $\boldsymbol{\beta}_W^{\top}\boldsymbol{\lambda}$, so prescribing a level of $P^{W}$ and prescribing the concept change $\rho$ of Theorem~\ref{thm:optimal_steering} are two coordinates for the same constraint; we report $P^{W}$ because it is directly interpretable and identical across methods.
Three concepts suffice to test what this paper claims. Whether geometric correction can reduce the local steering cost is settled theoretically rather than empirically: \appUnified{} proves that, under the second-order model, steering under a flat metric is never cheaper than steering under $G$, for any concept direction whatsoever, with equality only in a degenerate case. Adding concepts cannot strengthen or weaken that result. What experiments must establish is that the predicted advantage survives at practical step sizes, appears in the off-target component specifically, and is of the size the theory predicts---and the theory is specific about the size, attributing it to how the concept direction spreads across the spectrum of $G$. Section~\ref{subsec:exp_analysis} measures that quantity directly. The evidence here is therefore not three observations of one effect, but an effect that varies, across nearly two thousand cases and three model families, with a quantity named in advance.

\paragraph{Steering task and matched comparison.}
Each method steers each context toward four targets $P^{W} \in \{0.3, 0.5, 0.7, 0.9\}$. All comparisons are paired at matched concept change: a case enters a comparison only when both methods reach the same target on the same context, so off-target divergences are compared at equal on-target effect. Aggregation over targets takes the geometric mean of the four per-target ratios and requires all four targets to be reached; runs whose calibration fails to reach a target are excluded from that target's pairs, and per-cell counts appear in Table~\ref{tab:main} and \appStats.

\paragraph{Methods.}
FishBack reaches each target by iteration. At the current activation it solves the regularized system of \eqref{eq:optimal} in closed form, and applies the resulting direction at a fixed concept efficiency: the component along $\mathbf{q}$ is set to a constant and the remainder follows the orthogonal part of the solved direction (\appApprox). Fixing the concept efficiency guarantees progress toward the target when $G$ is ill-conditioned, and the same construction is used for the Euclidean ablation, which retains only the concept component, so the two runs differ exactly by the geometric term. The activation is then updated, and the covector and metric are re-evaluated at the new point---exactly at every step on GPT-2 Small, and through the fixed-Jacobian approximation described below on the 8B models---until the trajectory of $P^{W}$ has crossed the targets. Because the trajectory is discrete, each crossing is localized by bisection in activation space between the adjacent iterates, and the off-target divergence of \eqref{eq:offtarget} is evaluated at the localized crossing, so every method is scored at exactly the prescribed concept level. The same localization applies to the baselines, whose trajectories arise from sweeping the magnitude of their fixed directions. At the scale of an 8B model, exact computation is no longer viable: forming $G$ at $d=4096$ over a $10^{5}$-token vocabulary exceeds practical memory and time budgets by orders of magnitude. We therefore developed an approximation that computes the Fisher correction while avoiding the explicit matrix. The action of $G$ on a vector is evaluated by a finite-difference Jacobian--vector product followed by a backward product under automatic differentiation; the softmax Fisher is restricted to the roughly $10^{3}$ tokens carrying the predictive mass; and a randomized decomposition of rank $50$ \citep{halko2011finding} supplies the low-dimensional subspace in which the regularized inverse is taken, with the Jacobian held fixed across steps. Each component is validated in \appApprox. We compare against five steering methods: CAA \citep{rimsky2024steering}, ActAdd \citep{turner2023activation}, ITI \citep{li2023inference}, ReFT \citep{wu2024reft}, and Representation Surgery \citep{singh2024representation}. Two references complete the comparison: the \emph{Euclidean} ablation, which runs the same iterative pipeline with $G$ replaced by the identity and thus steers along $\mathbf{q}$ itself; and \emph{Best-per-case}, a per-case oracle that selects, for every case, the baseline with the smallest off-target divergence.

\paragraph{Statistics.}
Off-target comparisons are summarized by per-case ratios $\mathrm{KL}_{\mathrm{baseline}}/\mathrm{KL}_{\mathrm{FishBack}}$, so values above $1$ favor FishBack; we report medians with interquartile ranges. Significance uses one-sided paired Wilcoxon signed-rank tests on the log-ratios, each pair consisting of the two methods' off-target divergences for the same context at the same target, whose null distribution is symmetric where that of raw ratios is not, with Benjamini--Hochberg correction applied within families defined per model and aggregation level (\appStats); cells with fewer than six pairs are marked underpowered and excluded from confirmatory claims. Families, audits, and per-cell tables appear in \appStats.
\begin{figure}[t]
	\centering
	\includegraphics[width=\columnwidth]{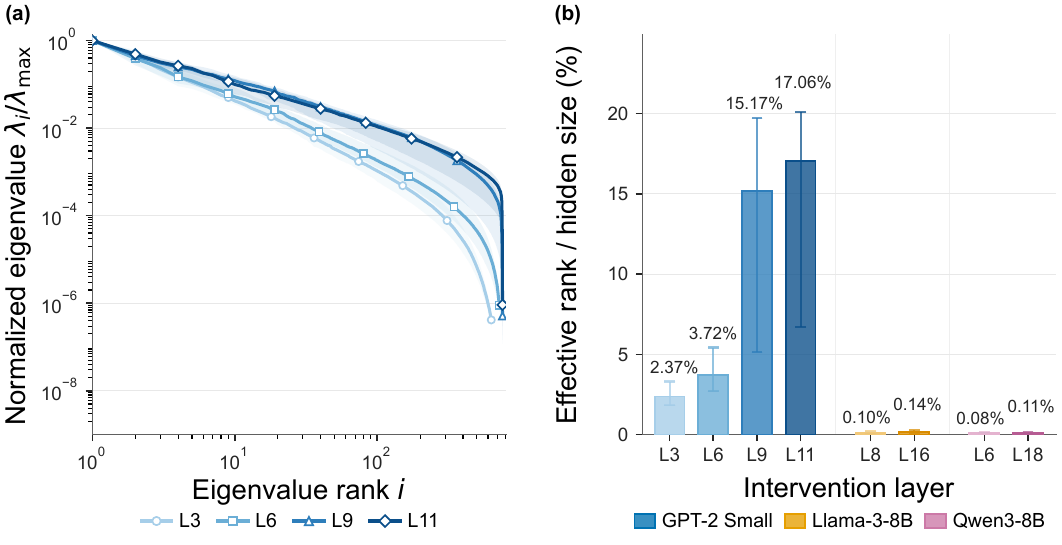}
	\caption{The pullback Fisher metric is strongly non-Euclidean at every layer studied. \textbf{Left:} normalized eigenvalue spectra of $G$, decaying by more than six orders of magnitude. \textbf{Right:} median effective rank as a fraction of the hidden dimension, with interquartile ranges over contexts.}
	\label{fig:geometry}
\end{figure}
\subsection{Geometry of the Intermediate Layers}
\label{subsec:exp_geometry}

Theorem~\ref{thm:pullback} makes the operative geometry explicit; the first experiment measures how far it departs from the flat metric that additive methods assume. On GPT-2 Small, where $G$ is computed exactly, the deviation from the closest scaled identity (the deviation identity of \appGeometry) exceeds $97\%$ in relative Frobenius norm at every layer, and the spectrum decays by more than six orders of magnitude within the first few hundred eigendirections (Figure~\ref{fig:geometry}, left): the median effective rank per layer spans $2$--$17\%$ of the hidden dimension across layers $3$--$11$, with condition numbers on the order of $10^{7}$. The two 8B models concentrate further still---at their early and middle layers the median effective rank, computed within the rank-$50$ subspace retained by the approximation, falls below $0.2\%$ of $d=4096$ (Figure~\ref{fig:geometry}, right), a concentration with two consequences realized below: the rank-$50$ subspace captures essentially all of the metric's mass, and the geometric correction has more room to act (Section~\ref{subsec:exp_ablation}).  The premise of Section~\ref{sec:pullback}---that steering cost is governed by a strongly anisotropic quadratic form---holds quantitatively at every layer used below.
\begin{figure}[t]
	\centering
	\includegraphics[width=\columnwidth]{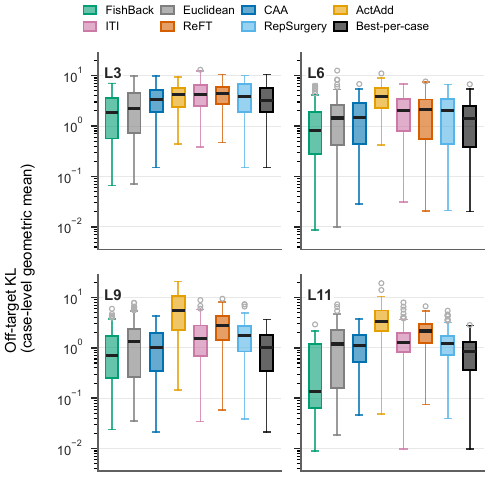}
	\caption{On GPT-2 Small, FishBack attains the lowest off-target KL divergence at matched concept change at every layer. Boxes show the median and interquartile range of per-case off-target KL (log scale), pooled over concepts and targets.}
	\label{fig:offtarget_gpt2}
\end{figure}
\begin{figure}[t]
	\centering
	\includegraphics[width=\columnwidth]{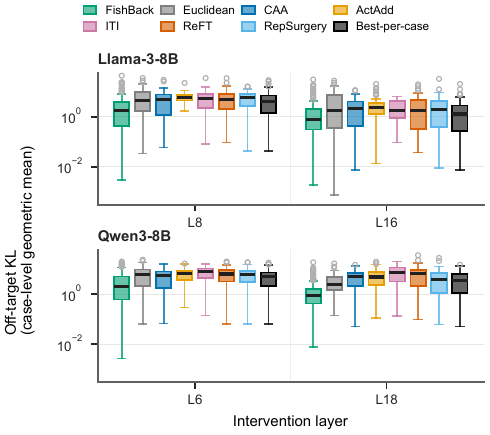}
	\caption{The off-target advantage persists on Llama-3-8B and Qwen3-8B at the early and middle layers under the matrix-free approximation. Boxes show per-case off-target KL (log scale) at matched concept change, with shared axes.}
	\label{fig:offtarget_8b}
\end{figure}
\subsection{Off-Target Comparison Against Baselines}
\label{subsec:exp_offtarget}
\begin{figure}[t]
	\centering
	\includegraphics[width=\columnwidth]{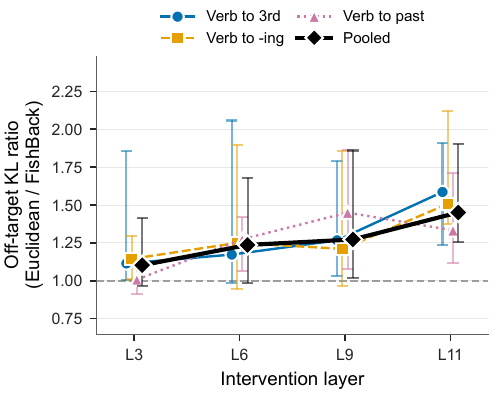}
	\caption{Ablating the metric isolates the geometric gain on GPT-2 Small: the Euclidean-to-FishBack off-target ratio lies above parity at every layer. Points are per-layer medians with interquartile ranges.}
	\label{fig:ablation_gpt2}
\end{figure}
Table~\ref{tab:main} reports the central comparison: the per-case ratio of each method's off-target divergence to FishBack's, at matched concept change, pooled over concepts, targets, and the layers in scope. On GPT-2 Small, where every step of \eqref{eq:optimal} is exact, FishBack attains a lower off-target divergence than each of the five steering methods, with median ratios from $1.42$ (CAA) to $6.52$ (ActAdd) and win rates of $69$--$94\%$; the per-cell tests remain significant after Benjamini--Hochberg correction in the large majority of concept--layer--target cells (\appStats). The advantage holds at every layer, including the deepest (Figure~\ref{fig:offtarget_gpt2}). Against the Best-per-case oracle, the median ratio is $1.25$ with a $64\%$ win rate---a thinner margin, and a bar no single fixed method faces in practice.
\begin{figure}[t]
	\centering
	\includegraphics[width=\columnwidth]{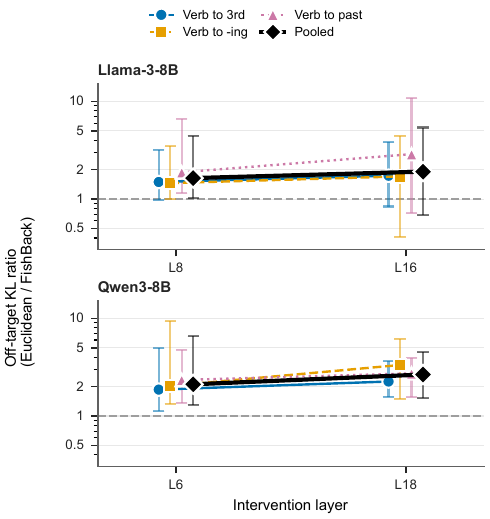}
	\caption{The geometric gain transfers to the 8B models: the Euclidean-to-FishBack off-target ratio stays above parity at the early and middle layers. Points are per-layer medians with interquartile ranges.}
	\label{fig:ablation_8b}
\end{figure}
The same protocol runs on Llama-3-8B and Qwen3-8B under the approximation of Section~\ref{subsec:exp_setup}. At the early and middle layers the advantage not only persists but widens: median ratios against individual baselines reach $2.2$--$3.5$ on Llama-3-8B and $1.8$--$3.6$ on Qwen3-8B, and the Best-per-case margins remain at $1.64$ and $1.49$ (Table~\ref{tab:main}; Figure~\ref{fig:offtarget_8b}). The deepest 8B layers, where the geometry itself contracts, are analyzed in Section~\ref{subsec:exp_analysis}. Because matched comparison conditions on both methods reaching a target, per-method reachability is reported in \appReach.

\subsection{Geometric Ablation}
\label{subsec:exp_ablation}

To attribute the advantage established above, the Euclidean ablation replaces the pullback Fisher metric $G$ in \eqref{eq:optimal} with the identity while keeping the iterative pipeline, the calibration, and the covector unchanged; the gap between the two runs therefore measures the contribution of the metric alone. Table~\ref{tab:main} (middle block) and Figures~\ref{fig:ablation_gpt2}--\ref{fig:ablation_8b} report this gap as the per-case ratio of Euclidean to FishBack off-target divergence. On GPT-2 Small the ratio sits above parity at every layer, with a pooled median of $1.30$ and an $80\%$ win rate; on the 8B models it widens to $1.79$ (Llama-3-8B) and $2.45$ (Qwen3-8B) at the early and middle layers.

\subsection{Analysis of the Advantage Regime}
\label{subsec:exp_analysis}

Section~\ref{subsec:computation} noted that the condition-number bound caps how anisotropic the geometry can become but does not by itself determine where correcting it pays. The analysis in \appUnified{} sharpens this: the attainable gain over Euclidean steering is bounded by the anisotropy of $G$, and it is realized only to the extent that the concept covector spreads across the spectrum of $G$ rather than aligning with a single eigendirection. Both quantities are measurable per case, which turns the framework into a prediction about where the ablation gap of Section~\ref{subsec:exp_ablation} should appear and where it should close.

The deepest 8B layers test the prediction. At Llama-3-8B L24 and Qwen3-8B L27 (roughly $75\%$ relative depth), the conditioning of $G$ has flattened substantially---the condition number falls by a factor of three on Llama-3-8B and forty on Qwen3-8B relative to their early layers (\appGeometry)---which lowers the ceiling of \appUnified{} on the attainable gain, and the Euclidean gap of Table~\ref{tab:main} is no longer significant. The closure is not uniform noise: at these layers the per-case gain becomes strongly predictable from the concept--spectrum alignment, with case-level Spearman correlations rising from near zero at the early layers to $0.50$ (Llama-3-8B) and $0.62$ (Qwen3-8B) at the deepest ones (\appDepth). Where the geometry is rich, the correction helps broadly; where it flattens, the gain narrows to the cases the theory selects.
\section{Conclusion}
\label{sec:conclusion}

Activation steering has treated the intermediate activation space as flat, and its familiar instabilities follow from that choice. This paper identified the metric that the choice ignores: the Fisher information of the softmax head, pulled back through the network's Jacobian, which measures to second order how much an intervention at layer $\ell$ moves the model's prediction. From this metric we derived the steering direction that achieves a prescribed concept change at the least off-target distortion, and related the widely used additive methods to steering under a flat metric. A matrix-free approximation carries the computation to 8B-scale models at a few hundred partial forward passes per case. On three verb-morphology concepts across GPT-2 Small, Llama-3-8B, and Qwen3-8B, the derived direction lowers off-target KL divergence at matched concept change by median factors of $1.4$--$6.5$x against individual steering baselines on GPT-2 Small and $1.8$--$3.6$x on the 8B models at the early and middle layers, with the geometric ablation attributing the gain to the metric itself.

\bibliography{aaai2027}
\clearpage
\appendix
\section{Proofs of the Main Results}
\label{app:proofs}
Theorem and proposition numbers in the headings below refer to the main paper; results stated here for the first time are numbered within this appendix.
\subsection{Setting and Notation}
Notation follows the main paper. The softmax head maps the final representation $\boldsymbol{\lambda}\in\mathbb{R}^{d}$ to $p_{\boldsymbol{\lambda}}=\mathrm{softmax}(W_U\boldsymbol{\lambda})$ with unembedding rows $\mathbf{w}_y$; the log-partition of this exponential family is
\begin{equation}
	\label{eq:app:logpart}
	A(\boldsymbol{\lambda})=\log\sum_{y\in\mathcal{V}}\exp(\mathbf{w}_y^{\top}\boldsymbol{\lambda}),
	p_{\boldsymbol{\lambda}}(y)=\exp\!\big(\mathbf{w}_y^{\top}\boldsymbol{\lambda}-A(\boldsymbol{\lambda})\big).
\end{equation}
We write
\begin{equation}
	\begin{aligned}
		\bar{\mathbf{w}}_{\boldsymbol{\lambda}}&=\mathbb{E}_{p_{\boldsymbol{\lambda}}}[\mathbf{w}],
		&\qquad
		H&=\mathrm{Cov}_{p_{\boldsymbol{\lambda}}}(\mathbf{w}),\\
		G&=J^{\top}HJ,
		&\qquad
		\mathbf{q}&=J^{\top}\boldsymbol{\beta}_W,
	\end{aligned}
\end{equation}
where $\boldsymbol{\lambda}=f(\mathbf{h})$ and $J$ is the Jacobian of $f$. For a concept $W$ with counterfactual pairs $\{(y_i^0,y_i^1)\}_{i=1}^{n_W}$, the concept tokens are $\mathcal{Y}_W=\bigcup_i\{y_i^0,y_i^1\}$ and $\mathcal{Y}_W^{c}$ is the remainder; $D_{\mathrm{KL}}^{\mathrm{off}}$ denotes the divergence between the distributions renormalized over $\mathcal{Y}_W^{c}$, as in the main paper.

\subsection{Proof of Theorem 1 (Second-Order Steering Cost)}

\begin{lemma}[KL as a Bregman divergence]
	\label{lem:bregman}
	For all $\boldsymbol{\lambda}_0,\boldsymbol{\lambda}_1$,
	\begin{equation}
		D_{\mathrm{KL}}\!\left(p_{\boldsymbol{\lambda}_0}\,\middle\|\,p_{\boldsymbol{\lambda}_1}\right)
		=A(\boldsymbol{\lambda}_1)-A(\boldsymbol{\lambda}_0)
		-\nabla A(\boldsymbol{\lambda}_0)^{\top}\big(\boldsymbol{\lambda}_1-\boldsymbol{\lambda}_0\big).
	\end{equation}
\end{lemma}

\begin{proof}
	By \eqref{eq:app:logpart},
	\begin{equation}
		\log\frac{p_{\boldsymbol{\lambda}_0}(y)}{p_{\boldsymbol{\lambda}_1}(y)}
		=\mathbf{w}_y^{\top}\big(\boldsymbol{\lambda}_0-\boldsymbol{\lambda}_1\big)
		-A(\boldsymbol{\lambda}_0)+A(\boldsymbol{\lambda}_1).
	\end{equation}
	Taking the expectation under $p_{\boldsymbol{\lambda}_0}$ and using $\mathbb{E}_{p_{\boldsymbol{\lambda}_0}}[\mathbf{w}]=\nabla A(\boldsymbol{\lambda}_0)$, which is verified in the next proof, gives the claim.
\end{proof}

\begin{lemma}[Hessian equals the softmax Fisher]
	\label{lem:hessian}
	$\nabla A(\boldsymbol{\lambda})=\bar{\mathbf{w}}_{\boldsymbol{\lambda}}$ and $\nabla^{2}A(\boldsymbol{\lambda})=\mathrm{Cov}_{p_{\boldsymbol{\lambda}}}(\mathbf{w})=H$.
\end{lemma}

\begin{proof}
	Differentiating \eqref{eq:app:logpart},
	\begin{equation}
		\nabla A(\boldsymbol{\lambda})
		=\sum_y p_{\boldsymbol{\lambda}}(y)\,\mathbf{w}_y
		=\bar{\mathbf{w}}_{\boldsymbol{\lambda}}.
	\end{equation}
	Since $\nabla p_{\boldsymbol{\lambda}}(y)=p_{\boldsymbol{\lambda}}(y)\big(\mathbf{w}_y-\bar{\mathbf{w}}_{\boldsymbol{\lambda}}\big)$,
	\begin{align*}
		\nabla^{2}A(\boldsymbol{\lambda})
		&=\sum_y \nabla p_{\boldsymbol{\lambda}}(y)\,\mathbf{w}_y^{\top}
		=\sum_y p_{\boldsymbol{\lambda}}(y)\big(\mathbf{w}_y-\bar{\mathbf{w}}_{\boldsymbol{\lambda}}\big)\mathbf{w}_y^{\top}\\
		&=\sum_y p_{\boldsymbol{\lambda}}(y)\big(\mathbf{w}_y-\bar{\mathbf{w}}_{\boldsymbol{\lambda}}\big)\big(\mathbf{w}_y-\bar{\mathbf{w}}_{\boldsymbol{\lambda}}\big)^{\top},
	\end{align*}
	where the last step uses $\sum_y p_{\boldsymbol{\lambda}}(y)\big(\mathbf{w}_y-\bar{\mathbf{w}}_{\boldsymbol{\lambda}}\big)=\mathbf{0}$.
\end{proof}

\begin{proof}[Proof of Theorem 1]
	Let
	\begin{equation}
		\Delta\boldsymbol{\lambda}
		=f(\mathbf{h}_0+\delta\mathbf{h})-f(\mathbf{h}_0)
		=J\,\delta\mathbf{h}+O\!\big(\|\delta\mathbf{h}\|^{2}\big)
	\end{equation}
	for $f\in C^{2}$. By Lemma~\ref{lem:bregman} and a third-order Taylor expansion of $A$ at $\boldsymbol{\lambda}_0$, the first-order terms cancel and
	\begin{equation}
		D_{\mathrm{KL}}
		=\tfrac12\,\Delta\boldsymbol{\lambda}^{\top}H\,\Delta\boldsymbol{\lambda}
		+O\!\big(\|\Delta\boldsymbol{\lambda}\|^{3}\big),
	\end{equation}
	using Lemma~\ref{lem:hessian} for the quadratic coefficient. Substituting the expansion of $\Delta\boldsymbol{\lambda}$ and absorbing all cross and higher terms into the remainder yields
	\begin{equation}
		D_{\mathrm{KL}}
		=\tfrac12\,\delta\mathbf{h}^{\top}G\,\delta\mathbf{h}
		+O\!\big(\|\delta\mathbf{h}\|^{3}\big),
		\qquad
		G=J^{\top}HJ.
	\end{equation}
	Symmetry is immediate; positive semi-definiteness follows since
	$\delta\mathbf{h}^{\top}G\,\delta\mathbf{h}
	=(J\delta\mathbf{h})^{\top}H(J\delta\mathbf{h})\ge 0$
	by Lemma~\ref{lem:hessian}. An explicit bound on the third-order remainder, and the critical radius within which the quadratic model is reliable, are given in Appendix~D.
\end{proof}

\subsection{Proof of Proposition 1 (Gauss--Newton Form)}

\begin{proof}
	The negative log-likelihood of a label $y$ as a function of the final representation is
	\begin{equation}
		\mathcal{L}_y(\boldsymbol{\lambda})=A(\boldsymbol{\lambda})-\mathbf{w}_y^{\top}\boldsymbol{\lambda},
	\end{equation}
	whose Hessian $\nabla^{2}_{\boldsymbol{\lambda}}\mathcal{L}_y=\nabla^{2}A=H$ is independent of $y$ by Lemma~\ref{lem:hessian}. The generalized Gauss--Newton matrix of $\mathcal{L}_y\circ f$ with respect to $\mathbf{h}$ discards the term involving second derivatives of $f$ and retains
	\begin{equation}
		J^{\top}\big(\nabla^{2}_{\boldsymbol{\lambda}}\mathcal{L}_y\big)J
		=J^{\top}HJ=G
	\end{equation}
	for every label.
\end{proof}

\subsection{Proof of Theorem 2 (Optimal Steering Direction)}

Throughout this subsection assume $\mathbf{q}\in\mathrm{Range}(G)$ and $\mathbf{q}\neq\mathbf{0}$, so that $\mathbf{q}^{\top}G^{\dagger}\mathbf{q}>0$; in practice this is secured by the spectral regularization of the main paper, which replaces $G$ by a positive-definite $G+\alpha I$ (Appendix~D).
\begin{proof}[Proof of Theorem 2]
	The Lagrangian
	\begin{equation}
		\mathcal{L}(\delta\mathbf{h},\nu)
		=\tfrac12\,\delta\mathbf{h}^{\top}G\,\delta\mathbf{h}
		-\nu\big(\mathbf{q}^{\top}\delta\mathbf{h}-\rho\big)
	\end{equation}
	has stationarity condition
	\begin{equation}
		G\,\delta\mathbf{h}=\nu\,\mathbf{q}.
	\end{equation}
	Since $\mathbf{q}\in\mathrm{Range}(G)$, all solutions are $\delta\mathbf{h}=\nu\,G^{\dagger}\mathbf{q}+\mathbf{k}$ with $\mathbf{k}\in\ker(G)$, and the constraint fixes
	\begin{equation}
		\nu\,\mathbf{q}^{\top}G^{\dagger}\mathbf{q}+\mathbf{q}^{\top}\mathbf{k}=\rho.
	\end{equation}
	Because the objective is convex ($G\succeq 0$) and constant along $\ker(G)$, the set of minimizers is the affine family
	\begin{equation}
		\delta\mathbf{h}^{\star}+\big(\ker(G)\cap\mathbf{q}^{\perp}\big),
		\qquad
		\delta\mathbf{h}^{\star}
		=\frac{\rho}{\mathbf{q}^{\top}G^{\dagger}\mathbf{q}}\,G^{\dagger}\mathbf{q},
	\end{equation}
	all attaining the optimal cost
	\begin{equation}
		\frac{\rho^{2}}{2\,\mathbf{q}^{\top}G^{\dagger}\mathbf{q}}.
	\end{equation}
	Among them, $\delta\mathbf{h}^{\star}$ is the unique minimum-norm element, since $G^{\dagger}\mathbf{q}\in\mathrm{Range}(G)\perp\ker(G)$. Kernel directions carry no second-order cost but are not certified by the quadratic model beyond second order; the minimum-norm choice discards them. Setting $G=I$ recovers $\delta\mathbf{h}^{\star}=(\rho/\|\mathbf{q}\|^{2})\,\mathbf{q}$, the additive step of the flat special case.
\end{proof}

\subsection{Proof of Theorem 3 (Non-Target Optimality)}

\begin{definition}[Concept decomposability]
	\label{def:decomp}
	Write $\mathbf{d}_i=\mathbf{w}_{y_i^{1}}-\mathbf{w}_{y_i^{0}}$. The distribution $p_{\boldsymbol{\lambda}}$ is \emph{decomposable} for $W$ if there exist a binary distribution $P^{W}_{\boldsymbol{\lambda}}$ and a distribution $P^{Z}_{\boldsymbol{\lambda}}$ over contents---one atom $z_i$ per counterfactual pair and one atom per token of $\mathcal{Y}_W^{c}$---such that
	\begin{equation}
		\begin{aligned}
			p_{\boldsymbol{\lambda}}(y_i^{w})&=P^{W}_{\boldsymbol{\lambda}}(w)\,P^{Z}_{\boldsymbol{\lambda}}(z_i)
			&&\text{for all } i,w,\\
			p_{\boldsymbol{\lambda}}(y)&=P^{Z}_{\boldsymbol{\lambda}}(z_y)
			&&\text{for } y\in\mathcal{Y}_W^{c}.
		\end{aligned}
	\end{equation}
\end{definition}

\begin{lemma}[Criterion]
	\label{lem:criterion}
	$p_{\boldsymbol{\lambda}}$ is decomposable for $W$ if and only if $\mathbf{d}_i^{\top}\boldsymbol{\lambda}$ is the same for all $i$, i.e.
	\begin{equation}
		(\mathbf{d}_i-\mathbf{d}_j)^{\top}\boldsymbol{\lambda}=0
		\qquad\text{for all } i,j.
	\end{equation}
\end{lemma}

\begin{proof}
	Decomposability forces the within-pair odds
	\begin{equation}
		\frac{p_{\boldsymbol{\lambda}}(y_i^{1})}{p_{\boldsymbol{\lambda}}(y_i^{0})}
		=\frac{P^{W}(1)}{P^{W}(0)}
	\end{equation}
	to be independent of $i$; by \eqref{eq:app:logpart} this odds equals $\exp(\mathbf{d}_i^{\top}\boldsymbol{\lambda})$. Conversely, if all $\mathbf{d}_i^{\top}\boldsymbol{\lambda}$ are equal, define
	\begin{equation}
		\begin{aligned}
			P^{W}(1)&=\sigma\!\big(\mathbf{d}_1^{\top}\boldsymbol{\lambda}\big),\\
			P^{Z}(z_i)&=p_{\boldsymbol{\lambda}}(y_i^{0})+p_{\boldsymbol{\lambda}}(y_i^{1}),\\
			P^{Z}(z_y)&=p_{\boldsymbol{\lambda}}(y);
		\end{aligned}
	\end{equation}
	the factorization follows.
\end{proof}

\begin{theorem}[Non-target optimality, precise form]
	\label{thm:app:nontarget}
	Let $\mathcal{A}_\rho$ be the feasible set of interventions achieving the prescribed concept change, and suppose every distribution reached from $\mathcal{A}_\rho\cup\{\mathbf{h}_0\}$ is decomposable for $W$, with $P^{W}$ determined by the concept score $\boldsymbol{\beta}_W^{\top}\boldsymbol{\lambda}$. Then
	\begin{equation}
		\arg\min_{\delta\mathbf{h}\in\mathcal{A}_\rho} D_{\mathrm{KL}}\!\left(p_{\boldsymbol{\lambda}_0}\,\middle\|\,p_{\boldsymbol{\lambda}}\right)
		\;\subseteq\;
		\arg\min_{\delta\mathbf{h}\in\mathcal{A}_\rho} D_{\mathrm{KL}}\!\left(P^{Z}_{\boldsymbol{\lambda}_0}\,\middle\|\,P^{Z}_{\boldsymbol{\lambda}}\right).
	\end{equation}
\end{theorem}

\begin{proof}
	Expanding the total divergence over $\mathcal{Y}_W\cup\mathcal{Y}_W^{c}$ under Definition~\ref{def:decomp} and collecting terms,
	\begin{equation}
		\begin{aligned}
			D_{\mathrm{KL}}\!\left(p_{\boldsymbol{\lambda}_0}\,\middle\|\,p_{\boldsymbol{\lambda}}\right)
			={}&\Big(\textstyle\sum_i P^{Z}_{\boldsymbol{\lambda}_0}(z_i)\Big)\,
			D_{\mathrm{KL}}\!\left(P^{W}_{\boldsymbol{\lambda}_0}\,\middle\|\,P^{W}_{\boldsymbol{\lambda}}\right)\\
			&+D_{\mathrm{KL}}\!\left(P^{Z}_{\boldsymbol{\lambda}_0}\,\middle\|\,P^{Z}_{\boldsymbol{\lambda}}\right),
		\end{aligned}
	\end{equation}
	where the cross terms recombine into the second divergence because $\sum_w P^{W}(w)=1$. On $\mathcal{A}_\rho$ the concept score is fixed, hence $P^{W}_{\boldsymbol{\lambda}}$ is constant and the first term does not vary with $\delta\mathbf{h}$; minimizing the left side over $\mathcal{A}_\rho$ therefore minimizes the second term.
\end{proof}
Theorem~3 of the main paper is this statement specialized to the reported metric: under the same condition, the direction that minimizes the total divergence also minimizes the non-concept divergence, whose dominant component is $D_{\mathrm{KL}}^{\mathrm{off}}$ by Remark~\ref{rem:offtarget}.
\begin{remark}[Relation to the experimental metric]
	\label{rem:offtarget}
	Splitting $P^{Z}$ into its pair-content and $\mathcal{Y}_W^{c}$ components by the chain rule,
	\begin{equation}
		D_{\mathrm{KL}}\!\left(P^{Z}_{\boldsymbol{\lambda}_0}\,\middle\|\,P^{Z}_{\boldsymbol{\lambda}}\right)
		= d\!\left(m_0\,\middle\|\,m\right)
		+ m_0\, D_{\mathrm{KL}}^{\mathrm{pair}}
		+ (1-m_0)\, D_{\mathrm{KL}}^{\mathrm{off}},
	\end{equation}
	where $m$ is the content mass on the pairs and $d(\cdot\|\cdot)$ the binary divergence. The reported metric $D_{\mathrm{KL}}^{\mathrm{off}}$ is thus the dominant component whenever the concept tokens carry a small share of the mass, which holds throughout our evaluation contexts.
\end{remark}

\begin{remark}[Scope of the assumption]
	\label{rem:scope}
	Decomposability need only hold on the steering neighborhood. If all pairs share one difference vector, $\mathbf{d}_i\equiv\boldsymbol{\beta}_W$, the criterion of Lemma~\ref{lem:criterion} holds identically. Otherwise it holds approximately when
	\begin{equation}
		\zeta(r)=\sup_{\mathbf{h}\in B_r(\mathbf{h}_0)}\;\max_{i,j}\;
		\big|(\mathbf{d}_i-\mathbf{d}_j)^{\top}f(\mathbf{h})\big|
	\end{equation}
	is small, in which case all within-pair odds agree up to a factor $e^{\pm 2\zeta}$. Exactness at one point does not extend automatically: with $\mathbf{d}_a=(1,0)$, $\mathbf{d}_b=(0,1)$ and the reachable curve $f(t)=(t,2t)$, decomposability holds at $t=0$ and fails for every $t\neq 0$.
\end{remark}

\section{Structure of the Pullback Metric}
\label{app:structure}

This appendix develops the layerwise structure referenced in Section~4.4 of the main paper: the recursion relating the metrics at adjacent layers, the bounds on conditioning and effective rank that follow from it, and the behavior of the null space. Throughout, $G^{(\ell)}$ denotes the pullback metric at layer $\ell$, so that $G^{(L)}=H$ is the softmax Fisher at the final representation, and
\begin{equation}
	A_\ell=\frac{\partial \mathbf{h}^{(\ell+1)}}{\partial \mathbf{h}^{(\ell)}}
\end{equation}
is the Jacobian of the map from layer $\ell$ to layer $\ell+1$.

\subsection{Layerwise Recursion}

\begin{theorem}[Recursive decomposition]
	\label{thm:app:recursive}
	For every $\ell$,
	\begin{equation}
		\label{eq:app:recursive}
		G^{(\ell)}=A_\ell^{\top}G^{(\ell+1)}A_\ell .
	\end{equation}
\end{theorem}

\begin{proof}
	The chain rule gives $J_{\ell\to L}=J_{\ell+1\to L}A_\ell$, so
	\begin{align*}
		G^{(\ell)}
		&=J_{\ell\to L}^{\top}H\,J_{\ell\to L}\\
		&=A_\ell^{\top}J_{\ell+1\to L}^{\top}H\,J_{\ell+1\to L}A_\ell\\
		&=A_\ell^{\top}G^{(\ell+1)}A_\ell .
	\end{align*}
\end{proof}

The metric at each layer is thus a congruence transform of the metric above it, and the metric at an early layer is the output Fisher conjugated by the entire stack of intervening maps. Mild per-layer distortion accumulates multiplicatively, which is the mechanism behind the bounds below.

\subsection{Conditioning}

\begin{theorem}[Condition-number bound]
	\label{thm:app:condition}
	Suppose $A_\ell$ is invertible on the relevant subspace. Then
	\begin{equation}
		\label{eq:app:cond_step}
		\kappa\big(G^{(\ell)}\big)\le\kappa\big(G^{(\ell+1)}\big)\,\kappa(A_\ell)^{2}.
	\end{equation}
	If moreover $A_k=I+B_k$ with $\lVert B_k\rVert_2\le\rho_k<1$ for all $k$, then
	\begin{equation}
		\label{eq:app:cond_prod}
		\kappa\big(G^{(\ell)}\big)
		\le
		\kappa\big(G^{(L)}\big)
		\prod_{k=\ell}^{L-1}\left(\frac{1+\rho_k}{1-\rho_k}\right)^{2}.
	\end{equation}
\end{theorem}

\begin{proof}
	By Theorem~\ref{thm:app:recursive}, $\mathbf{v}^{\top}G^{(\ell)}\mathbf{v}=(A_\ell\mathbf{v})^{\top}G^{(\ell+1)}(A_\ell\mathbf{v})$ for every $\mathbf{v}$. Writing $m_{\ell+1}$ and $M_{\ell+1}$ for the extreme eigenvalues of $G^{(\ell+1)}$ on that subspace,
	\begin{align*}
		\lambda_{\max}\big(G^{(\ell)}\big)&\le M_{\ell+1}\,\sigma_{\max}(A_\ell)^{2},\\
		\lambda_{\min}\big(G^{(\ell)}\big)&\ge m_{\ell+1}\,\sigma_{\min}(A_\ell)^{2},
	\end{align*}
	and dividing gives \eqref{eq:app:cond_step}. For $A_k=I+B_k$ one has $\sigma_{\max}(A_k)\le 1+\rho_k$ and $\sigma_{\min}(A_k)\ge 1-\rho_k$, so $\kappa(A_k)\le(1+\rho_k)/(1-\rho_k)$; telescoping \eqref{eq:app:cond_step} from $L$ down to $\ell$ yields \eqref{eq:app:cond_prod}.
\end{proof}

\begin{remark}[Scope of the bounds]
	\label{rem:app:cond_scope}
	Equation \eqref{eq:app:cond_prod} is an upper bound. It permits the conditioning to degrade geometrically in the number of layers traversed, but does not force it: a layer with $\kappa(A_\ell)$ near one leaves the bound essentially unchanged, and the bound is blind to cancellations between layers. The measured condition numbers of \appGeometry{} do grow toward the shallower layers on all three models, though not in every concept-level cell. Corollary~\ref{cor:app:erank} is likewise specific to the trace effective rank $r_{\mathrm{tr}}$; the quantity measured in our experiments is the entropy effective rank of Definition~\ref{def:app:erank_entropy}, which weights the spectrum differently, so the two track each other qualitatively but the bound does not apply to the reported values verbatim. These bounds constrain how anisotropic the metric can become; they do not identify where geometric correction is most useful, which is the empirical question taken up in \appDepth.
\end{remark}

\subsection{Effective Rank}

\begin{definition}[Trace effective rank]
	\label{def:app:erank}
	For $G\succeq 0$ let
	\begin{equation}
		r_{\mathrm{tr}}(G)=\frac{\operatorname{tr}(G)}{\lambda_{\max}(G)},
		\qquad
		1\le r_{\mathrm{tr}}(G)\le\operatorname{rank}(G).
	\end{equation}
\end{definition}

\begin{corollary}[Decay bound]
	\label{cor:app:erank}
	If $A_\ell$ is invertible, then
	\begin{equation}
		r_{\mathrm{tr}}\big(G^{(\ell)}\big)\ge\kappa(A_\ell)^{-2}\,r_{\mathrm{tr}}\big(G^{(\ell+1)}\big),
	\end{equation}
	and consequently
	\begin{equation}
		r_{\mathrm{tr}}\big(G^{(\ell)}\big)
		\ge
		r_{\mathrm{tr}}\big(G^{(L)}\big)\prod_{k=\ell}^{L-1}\kappa(A_k)^{-2}.
	\end{equation}
\end{corollary}

\begin{proof}
	Cyclicity of the trace and $A_\ell A_\ell^{\top}\succeq\sigma_{\min}(A_\ell)^{2}I$ give
	\begin{equation*}
		\operatorname{tr}\big(G^{(\ell)}\big)
		=\operatorname{tr}\big(G^{(\ell+1)}A_\ell A_\ell^{\top}\big)
		\ge\sigma_{\min}(A_\ell)^{2}\operatorname{tr}\big(G^{(\ell+1)}\big),
	\end{equation*}
	while $\lambda_{\max}(G^{(\ell)})\le\sigma_{\max}(A_\ell)^{2}\lambda_{\max}(G^{(\ell+1)})$ as in the previous proof. Dividing the two gives the claim.
\end{proof}

\subsection{Null Space and Near-Null Space}

\begin{theorem}[Kernel recursion]
	\label{thm:app:kernel}
	For $G^{(\ell+1)}\succeq 0$,
	\begin{equation}
		\ker G^{(\ell)}=\big\{\mathbf{v}:A_\ell\mathbf{v}\in\ker G^{(\ell+1)}\big\}.
	\end{equation}
	In particular $\ker A_\ell\subseteq\ker G^{(\ell)}$, and if $A_\ell$ is invertible then $\dim\ker G^{(\ell)}=\dim\ker G^{(\ell+1)}$.
\end{theorem}

\begin{proof}
	For $G\succeq 0$ one has $\mathbf{x}^{\top}G\mathbf{x}=0$ if and only if $G\mathbf{x}=0$. Applying this to $\mathbf{v}^{\top}G^{(\ell)}\mathbf{v}=(A_\ell\mathbf{v})^{\top}G^{(\ell+1)}(A_\ell\mathbf{v})$ gives the characterization; the two consequences are immediate.
\end{proof}

\begin{remark}[Layer normalization does not enlarge the exact kernel]
	\label{rem:app:layernorm}
	The Jacobian of layer normalization annihilates the two-dimensional subspace spanned by the all-ones vector and the centered input. In a pre-normalization residual block, however, the layer map has Jacobian $A_k=I+B_kJ_{\mathrm{LN}}$, so any $\mathbf{v}\in\ker J_{\mathrm{LN}}$ satisfies $A_k\mathbf{v}=\mathbf{v}$ and is not annihilated. By Theorem~\ref{thm:app:kernel}, exact nullity therefore does not accumulate merely because normalization layers are traversed.
\end{remark}

\begin{theorem}[Threshold null space]
	\label{thm:app:threshold}
	Let $\mathcal{N}_\tau(G)=\operatorname{span}\{\mathbf{u}_i:\lambda_i(G)\le\tau\}$, and suppose $\sigma_{\min}(A_\ell)\ge s_\ell$ and $\sigma_{\max}(A_\ell)\le S_\ell$. Then, by the Courant--Fischer principle,
	\begin{equation}
		s_\ell^{2}\,\lambda_i\big(G^{(\ell+1)}\big)
		\le
		\lambda_i\big(G^{(\ell)}\big)
		\le
		S_\ell^{2}\,\lambda_i\big(G^{(\ell+1)}\big),
	\end{equation}
	and hence
	\begin{equation}
		\lambda_i\big(G^{(\ell+1)}\big)\le\frac{\tau}{S_\ell^{2}}
		\quad\Longrightarrow\quad
		\lambda_i\big(G^{(\ell)}\big)\le\tau .
	\end{equation}
\end{theorem}

When the intervening maps contract, $S_\ell<1$ and the set of directions falling below the threshold at the shallower layer is strictly larger than at the deeper one. The concentration observed at early layers is therefore a statement about the numerical near-null space rather than about exact nullity, and we describe it as such throughout.

\section{Geometry Measurements}
\label{app:geometry}

This appendix gives the identity behind the deviation figures quoted in Section~5.2 of the main paper, and reports the measured geometry at every layer of the three models.

\subsection{Deviation from the Euclidean Metric}

The Euclidean assumption amounts to modelling $G$ by a scalar multiple of the identity. The quality of the best such model is controlled by a single spectral quantity.

\begin{definition}[Participation rank]
	\label{def:app:pr}
	For $G\succeq 0$ let
	\begin{equation}
		\operatorname{PR}(G)=\frac{\big(\operatorname{tr}G\big)^{2}}{\lVert G\rVert_F^{2}}
		=\frac{\big(\sum_i\lambda_i\big)^{2}}{\sum_i\lambda_i^{2}},
	\end{equation}
	which satisfies $1\le\operatorname{PR}(G)\le\operatorname{rank}(G)\le d$.
\end{definition}

\begin{theorem}[Best scalar Euclidean approximation]
	\label{thm:app:deviation}
	For any $G\succeq 0$ in dimension $d$,
	\begin{equation}
		\label{eq:app:deviation}
		\min_{c\in\mathbb{R}}\frac{\lVert G-cI\rVert_F}{\lVert G\rVert_F}
		=\sqrt{1-\frac{\operatorname{PR}(G)}{d}},
	\end{equation}
	attained at $c^{\star}=\operatorname{tr}(G)/d$.
\end{theorem}

\begin{proof}
	Expanding the objective and using $\lVert I\rVert_F^{2}=d$,
	\begin{equation*}
		\lVert G-cI\rVert_F^{2}
		=\lVert G\rVert_F^{2}-2c\operatorname{tr}G+c^{2}d ,
	\end{equation*}
	a quadratic in $c$ minimized at $c^{\star}=\operatorname{tr}(G)/d$. Substituting,
	\begin{align*}
		\min_c\lVert G-cI\rVert_F^{2}
		&=\lVert G\rVert_F^{2}-\frac{(\operatorname{tr}G)^{2}}{d}\\
		&=\lVert G\rVert_F^{2}\left(1-\frac{\operatorname{PR}(G)}{d}\right),
	\end{align*}
	and dividing by $\lVert G\rVert_F^{2}$ and taking square roots gives \eqref{eq:app:deviation}.
\end{proof}

The identity makes the observed deviations a direct consequence of spectral concentration: a relative error of $0.98$ requires $\operatorname{PR}(G)/d\approx 0.04$, so no scalar metric can approximate $G$ once its spectrum occupies a small fraction of the ambient dimension. The values in Table~\ref{tab:app:gpt2geom} correspond to participation ranks between roughly $6$ and $34$ out of $768$.

\subsection{Measured Geometry}
\begin{definition}[Effective rank]
	\label{def:app:erank_entropy}
	Let $\tilde\lambda_i=\lambda_i/\sum_j\lambda_j$ be the normalized eigenvalues of $G$. The effective rank reported below is the exponentiated Shannon entropy of this distribution,
	\begin{equation}
		r_{\mathrm{eff}}(G)=\exp\Big(-\sum_i\tilde\lambda_i\log\tilde\lambda_i\Big),
	\end{equation}
	which equals $d$ for a flat spectrum and approaches $1$ as the mass concentrates on a single direction.
\end{definition}
Table~\ref{tab:app:gpt2geom} reports the exact geometry of GPT-2 Small, computed from the full $768$-dimensional spectrum at $100$ contexts per layer. The condition number grows as the metric is pulled back through more layers, and the effective rank contracts in the same direction, from $17\%$ of the hidden dimension at layer $11$ to $2.4\%$ at layer $3$; the deviation from the closest scalar metric exceeds $97\%$ everywhere.

\begin{table}[t]
	\centering
	\small
	\begin{tabular}{lccccc}
		\toprule
		Layer & $n$ & $\kappa(G)$ & $r_{\mathrm{eff}}$ [IQR] & $r_{\mathrm{eff}}/d$ & Dev. \\
		\midrule
		3  & 100 & $9.5{\times}10^{7}$ & 18.2 [14.0, 25.4] & 2.4\%  & 0.996 \\
		6  & 100 & $7.1{\times}10^{7}$ & 28.6 [20.7, 41.7] & 3.7\%  & 0.994 \\
		9  & 100 & $5.0{\times}10^{7}$ & 116.5 [39.3, 151.5] & 15.2\% & 0.978 \\
		11 & 100 & $4.9{\times}10^{7}$ & 131.0 [51.7, 154.3] & 17.1\% & 0.980 \\
		\bottomrule
	\end{tabular}
	\caption{Exact pullback-Fisher geometry of GPT-2 Small ($d=768$), medians over $100$ contexts per layer. Dev.\ is the relative Frobenius deviation \eqref{eq:app:deviation} from the closest scalar metric.}
	\label{tab:app:gpt2geom}
\end{table}

At $8$B scale the spectrum is obtained from the rank-$50$ randomized decomposition of \appApprox, so the quantities in Table~\ref{tab:app:geom8b} describe the retained dominant subspace rather than the full spectrum; the deviation \eqref{eq:app:deviation} requires the complete spectrum and is therefore reported only for GPT-2 Small. Within that subspace the concentration is far sharper than on GPT-2 Small, with effective ranks near $0.1\%$ of $d=4096$, and the condition number again grows toward the shallower layers on both models.

\begin{table}[t]
	\centering
	\small
	\begin{tabular}{llcccc}
		\toprule
		Model & Layer & $n$ & $\kappa(G)$ & $r_{\mathrm{eff}}$ [IQR] & $r_{\mathrm{eff}}/d$ \\
		\midrule
		Llama-3-8B & 8  & 126 & $1.2{\times}10^{4}$ & 4.25 [2.58, 7.67] & 0.10\% \\
		& 16 & 117 & $4.0{\times}10^{3}$ & 5.63 [2.91, 11.22] & 0.14\% \\
		& 24 & 93  & $3.9{\times}10^{3}$ & 4.45 [2.60, 11.96] & 0.11\% \\
		\midrule
		Qwen3-8B   & 6  & 291 & $1.2{\times}10^{5}$ & 3.15 [2.06, 5.14] & 0.08\% \\
		& 18 & 292 & $1.8{\times}10^{4}$ & 4.54 [2.28, 6.64] & 0.11\% \\
		& 27 & 289 & $2.9{\times}10^{3}$ & 6.07 [2.91, 9.67] & 0.15\% \\
		\bottomrule
	\end{tabular}
	\caption{Pullback-Fisher geometry of the two $8$B models ($d=4096$), medians over successful cases with finite spectra, computed within the rank-$50$ subspace retained by the approximation of \appApprox.}
	\label{tab:app:geom8b}
\end{table}

The two tables together support the premise of the main paper's analysis at every layer used in the comparisons: the metric is far from any scalar multiple of the identity, and the directions along which an intervention actually moves the prediction number in the single digits on the 8B models and in the tens to low hundreds on GPT-2 Small, against ambient dimensions of $4096$ and $768$.

\section{Regularization and Validity of the Linear Solve}
\label{app:regularization}

The direction of Theorem~2 requires inverting a metric that is rank-deficient and ill-conditioned, and it is derived from a second-order model of a nonlinear map. This appendix bounds the bias introduced by regularization, gives the radius within which the second-order model is faithful, bounds the effect of an inexact metric, and reports a controlled numerical study.

\subsection{Regularized Solve and Its Bias}

Write $\delta=G^{\dagger}\mathbf{q}$ for the exact solution and
\begin{equation}
	\delta_\alpha=(G+\alpha I)^{-1}\mathbf{q}
\end{equation}
for the regularized one.

\begin{theorem}[Regularization bias]
	\label{thm:app:ridge}
	Suppose $\mathbf{q}\in\operatorname{Range}(G)$ and let $m>0$ be the smallest positive eigenvalue of $G$. Then
	\begin{equation}
		\lVert\delta_\alpha-\delta\rVert\le\frac{\alpha}{m(m+\alpha)}\,\lVert\mathbf{q}\rVert .
	\end{equation}
\end{theorem}

\begin{proof}
	Let $G=\sum_{i=1}^{r}\lambda_i\mathbf{u}_i\mathbf{u}_i^{\top}$ with $\lambda_i\ge m$, and expand $\mathbf{q}=\sum_i q_i\mathbf{u}_i$, which is possible because $\mathbf{q}\in\operatorname{Range}(G)$. Then
	\begin{equation*}
		\delta=\sum_i\frac{q_i}{\lambda_i}\mathbf{u}_i,
		\qquad
		\delta_\alpha=\sum_i\frac{q_i}{\lambda_i+\alpha}\mathbf{u}_i,
	\end{equation*}
	so that
	\begin{align*}
		\lVert\delta_\alpha-\delta\rVert^{2}
		&=\sum_i q_i^{2}\left(\frac{\alpha}{\lambda_i(\lambda_i+\alpha)}\right)^{2}\\
		&\le\left(\frac{\alpha}{m(m+\alpha)}\right)^{2}\lVert\mathbf{q}\rVert^{2}.
	\end{align*}
\end{proof}

The bias is therefore controlled as long as $\alpha$ is small relative to the smallest eigenvalue that carries signal. Choosing $\alpha$ well requires it to track the scale of the spectrum, which also makes the solution independent of the units in which activations are expressed.

\begin{proposition}[Scale equivariance]
	\label{prop:app:scale}
	Consider the reparameterization $\tilde{\mathbf{h}}=t\mathbf{h}$ of the activation coordinates, under which $\tilde{G}=t^{-2}G$ and $\tilde{\mathbf{q}}=t^{-1}\mathbf{q}$. If $\alpha$ is set to a fixed multiple of a degree-one spectral statistic of $G$, as in $\alpha=c\,\lambda_{\mathrm{med}}(G)$, then
	\begin{equation}
		\tilde\delta_{\tilde\alpha}=t\,\delta_\alpha ,
	\end{equation}
	which is the transformation satisfied by the exact solution. A fixed $\alpha$ does not satisfy it.
\end{proposition}

\begin{proof}
	With $\tilde\alpha=c\,\lambda_{\mathrm{med}}(\tilde G)=t^{-2}\alpha$,
	\begin{align*}
		\tilde\delta_{\tilde\alpha}
		&=\big(t^{-2}G+t^{-2}\alpha I\big)^{-1}\big(t^{-1}\mathbf{q}\big)\\
		&=t^{2}(G+\alpha I)^{-1}t^{-1}\mathbf{q}
		=t\,\delta_\alpha .
	\end{align*}
	The exact solution obeys $\tilde\delta=\tilde G^{\dagger}\tilde{\mathbf{q}}=t\,\delta$, so the regularized solve inherits the correct scaling. For fixed $\alpha$ the matrix $t^{-2}G+\alpha I$ is not proportional to $G+\alpha I$, and the identity fails.
\end{proof}

With condition numbers of order $10^{7}$ (\appGeometry), a fixed ridge either leaves the smallest eigenvalues under-regularized or suppresses the largest ones; an adaptive $\alpha$ avoids both.

\subsection{Validity Radius of the Second-Order Model}

Let $S$ be the subspace on which the solve is performed, with $G|_S\succeq mI_S$. Fix a reference radius $r_0>0$ and let $C_{\mathrm{KL}}$ bound the third-order remainder of the output KL on the ball of radius $r_0$, so that the remainder is at most $C_{\mathrm{KL}}\lVert\delta\mathbf{h}\rVert^{3}$ for $\lVert\delta\mathbf{h}\rVert\le r_0$.

\begin{theorem}[Critical radius]
	\label{thm:app:radius}
	The quadratic term satisfies $\tfrac12\delta\mathbf{h}^{\top}G\,\delta\mathbf{h}\ge\tfrac12 m\lVert\delta\mathbf{h}\rVert^{2}$. The third-order remainder is at most a fraction $\eta_{\mathrm{KL}}\in(0,1)$ of the quadratic term whenever
	\begin{equation}
		\lVert\delta\mathbf{h}\rVert\le r_{\mathrm{KL}}:=\min\!\Big\{r_0,\;\frac{\eta_{\mathrm{KL}}\,m}{2\,C_{\mathrm{KL}}}\Big\}.
	\end{equation}
\end{theorem}

\begin{proof}
	Substituting the two bounds into $C_{\mathrm{KL}}(r)\lVert\delta\mathbf{h}\rVert^{3}\le\eta_{\mathrm{KL}}\cdot\tfrac12 m\lVert\delta\mathbf{h}\rVert^{2}$ and dividing by $\lVert\delta\mathbf{h}\rVert^{2}$ gives the stated bound.
\end{proof}

An analogous requirement on the linearization of the concept constraint yields a second radius, and the operative one is the smaller of the two. These radii govern the accuracy of the KL \emph{value}; the accuracy of the \emph{direction} is a separate question, taken up numerically below.

\subsection{Stability Under an Inexact Metric}

The $8$B procedure of \appApprox{} replaces $G$ by an approximation and $\mathbf{q}$ by a computed covector. The resulting direction error is controlled.

\begin{theorem}[Perturbation bound]
	\label{thm:app:perturb}
	Let $\delta_\star=(G+\alpha I)^{-1}\mathbf{q}$ and $\tilde\delta=(M+\alpha I)^{-1}\hat{\mathbf{q}}$, and suppose
	\begin{equation}
		\lambda_{\min}(G+\alpha I)\ge m,
		\quad
		\lVert M-G\rVert_2\le\eta<m,
		\quad
		\lVert\hat{\mathbf{q}}-\mathbf{q}\rVert\le\xi .
	\end{equation}
	Then
	\begin{equation}
		\lVert\tilde\delta-\delta_\star\rVert
		\le\frac{\eta}{m(m-\eta)}\lVert\mathbf{q}\rVert+\frac{\xi}{m-\eta}.
	\end{equation}
\end{theorem}

\begin{proof}
	Write $M_\alpha=M+\alpha I$ and $G_\alpha=G+\alpha I$, and split
	\begin{equation*}
		\tilde\delta-\delta_\star
		=M_\alpha^{-1}(\hat{\mathbf{q}}-\mathbf{q})
		+\big(M_\alpha^{-1}-G_\alpha^{-1}\big)\mathbf{q}.
	\end{equation*}
	Weyl's inequality gives $\lambda_{\min}(M_\alpha)\ge m-\eta$, bounding the first term by $\xi/(m-\eta)$. The resolvent identity $M_\alpha^{-1}-G_\alpha^{-1}=M_\alpha^{-1}(G_\alpha-M_\alpha)G_\alpha^{-1}$ bounds the second by $\eta\lVert\mathbf{q}\rVert/\big(m(m-\eta)\big)$.
\end{proof}

\subsection{Numerical Study on a Synthetic Model}

To examine the direction beyond the radius of Theorem~\ref{thm:app:radius}, we use a synthetic softmax model in which $G$, the exact KL-minimizing direction, and $r_{\mathrm{KL}}$ can all be computed without approximation. Three configurations of increasing width are each instantiated in $300$ cases, at step sizes whose median exceeds the critical radius of Theorem~\ref{thm:app:radius} by factors between $82$ and $1183$. Table~\ref{tab:app:toy} reports, at a concept change of $\rho=0.1$, the cosine between the Fisher direction and the exact minimizer, the ratio of Euclidean to Fisher output KL, and the rank correlation between $\kappa(G)$ and that ratio.

\begin{table}[t]
	\centering
	\small
	\begin{tabular}{lcccc}
		\toprule
		Config. & $\cos(\delta,\delta_{\mathrm{exact}})$ & $\mathrm{KL}_{\mathrm{euc}}/\mathrm{KL}_{\mathrm{fish}}$ & $r_s(\kappa,\text{ratio})$ & $\lVert\delta\mathbf{h}\rVert/r_{\mathrm{KL}}$ \\
		\midrule
		tiny   & 0.997 & 2.39 [1.73, 4.46] & 0.29 & 1183 \\
		small  & 1.000 & 1.58 [1.35, 2.00] & 0.37 & 82 \\
		medium & 1.000 & 1.80 [1.51, 2.24] & 0.56 & 288 \\
		\midrule
		pooled & --- & 1.82 [1.46, 2.53] & 0.50 & 239 \\
		\bottomrule
	\end{tabular}
	\caption{Synthetic-model study at $\rho=0.1$; medians with interquartile ranges over $300$ cases per configuration. All rank correlations have $p<10^{-5}$, and the pooled correlation has $p<10^{-56}$. The last column reports how far the operating step size exceeds the critical radius of Theorem~\ref{thm:app:radius}.}
	\label{tab:app:toy}
\end{table}

Two observations follow. The Fisher direction stays close to the exact KL-minimizing direction even though the operating step size exceeds $r_{\mathrm{KL}}$ by a median factor between $82$ and $1183$, so the radius that governs the accuracy of the KL value is not the radius that governs the usefulness of the direction. And the advantage over Euclidean steering grows with the anisotropy of the metric, in the direction predicted by the bound of \appUnified.
\section{Matrix-Free Computation at 8B Scale}
\label{app:approximation}

The direction of Theorem~2 requires the regularized inverse of $G$ applied to $\mathbf{q}$. At $d=4096$ over a vocabulary of order $10^{5}$, assembling $J$ alone would take $d$ Jacobian evaluations at every step, which is out of reach. This appendix describes the procedure that produces the direction without ever forming $J$ or $G$, and states its error sources.

\subsection{Action of the Metric Without Forming It}

Every step of the computation needs only the map $\mathbf{v}\mapsto G\mathbf{v}$, which factors as
\begin{equation}
	G\mathbf{v}=J^{\top}\big(H(J\mathbf{v})\big)
\end{equation}
and is evaluated in three stages.

\emph{Forward.} $J\mathbf{v}$ is obtained by a central finite difference,
\begin{equation}
	J\mathbf{v}\approx\frac{f_\ell(\mathbf{h}+\varepsilon\mathbf{v})-f_\ell(\mathbf{h}-\varepsilon\mathbf{v})}{2\varepsilon},
	\qquad \varepsilon=10^{-4},
\end{equation}
at the cost of two partial forward passes. Forward-mode automatic differentiation would need one, but conflicts with the fused attention and normalization kernels used at this scale; the finite difference is numerically stable and reproducible.

\emph{Fisher.} The softmax Fisher is applied without materializing a $d\times d$ matrix. Retaining the $K=1000$ tokens of largest probability and renormalizing gives weights $\tilde{p}$ and rows $W_K$; with $\bar{\mathbf{w}}=\sum_i\tilde{p}_i\mathbf{w}_i$ and $\widetilde{W}=\operatorname{diag}(\sqrt{\tilde{p}})\,(W_K-\mathbf{1}\bar{\mathbf{w}}^{\top})$,
\begin{equation}
	H\mathbf{x}=\widetilde{W}^{\top}\big(\widetilde{W}\mathbf{x}\big),
\end{equation}
which costs $O(Kd)$ and touches only $K\times d$ storage. The predictive distribution at this scale is concentrated enough that $K=1000$ carries almost all of the Fisher mass.

\emph{Backward.} $J^{\top}\mathbf{u}=\nabla_{\mathbf{h}}\big(\mathbf{u}^{\top}f_\ell(\mathbf{h})\big)$ is exact and costs one partial backward pass.

A single metric-vector product therefore costs two forward and one backward pass, with no $d\times d$ object at any point.

\subsection{Low-Rank Spectrum}

Because the spectrum of $G$ is concentrated in far fewer than $50$ directions at every layer measured (\appGeometry), the regularized solve is performed inside a low-dimensional subspace obtained by randomized decomposition \citep{halko2011finding}. With target rank $r=50$ and oversampling $20$, a Gaussian test matrix $\Omega\in\mathbb{R}^{d\times 70}$ gives $Y=G\Omega$, whose orthonormal basis $Q$ is refined by forming the small matrix
\begin{equation}
	T=Q^{\top}(GQ)\in\mathbb{R}^{70\times 70},
\end{equation}
symmetrizing it, and diagonalizing. The procedure uses $140$ metric-vector products in total and returns the leading eigenpairs $(\lambda_j,\mathbf{u}_j)$. All subsequent inverses, and all spectral diagnostics reported in the paper, are computed from this representation.

\subsection{Applying the Direction at Fixed Concept Efficiency}

The regularized Fisher direction $(G+\alpha I)^{-1}\mathbf{q}$, with $\alpha$ set to the median positive eigenvalue as in Proposition~\ref{prop:app:scale}, can carry a small component along the concept covector when the spectrum is ill-conditioned, so that reaching a prescribed level of $P^{W}$ requires impractically large steps. We therefore apply it in a form whose concept efficiency is fixed. Writing $\hat{\mathbf{q}}=\mathbf{q}/\lVert\mathbf{q}\rVert$ and decomposing
\begin{equation}
	\delta_{\mathrm{F}}=(G+\alpha I)^{-1}\mathbf{q},
	\qquad
	\delta_\perp=\delta_{\mathrm{F}}-\big(\delta_{\mathrm{F}}^{\top}\hat{\mathbf{q}}\big)\hat{\mathbf{q}},
\end{equation}
the applied direction is
\begin{equation}
	\label{eq:app:cmin}
	\delta=c_{\min}\,\hat{\mathbf{q}}+\sqrt{1-c_{\min}^{2}}\;\frac{\delta_\perp}{\lVert\delta_\perp\rVert},
\end{equation}
oriented so that $\delta^{\top}\hat{\mathbf{q}}>0$, with $c_{\min}=0.5$ on the 8B models and $0.3$ on GPT-2 Small. If $\lVert\delta_\perp\rVert$ falls below $10^{-10}$ the direction reduces to $\hat{\mathbf{q}}$ and the case is flagged.

Two consequences are worth stating plainly. The first component of \eqref{eq:app:cmin} guarantees a fixed rate of progress toward the target, so reachability does not depend on the conditioning of $G$. The second component is where the pullback geometry acts: it is the part of the Fisher direction orthogonal to the concept, and it is what reduces movement of the rest of the distribution. The Euclidean ablation of the main paper is exactly \eqref{eq:app:cmin} with this second component removed, so the two runs share the same covector, the same calibration, and the same concept efficiency, and differ only in the geometric correction.

\subsection{Frozen Jacobian Across Steps}

Steering proceeds for $T=30$ steps, and in principle both $J$ and $H$ should be re-evaluated at each. We instead compute the randomized decomposition once, at the base point $\mathbf{h}_0$, and hold $J_0$ fixed, updating only the Fisher and the covector:
\begin{equation}
	G_t=J_0^{\top}H_t\,J_0 ,
\end{equation}
which is a recomposition inside the retained subspace and costs one top-$K$ Fisher evaluation and one backward pass per step. Recomputing the Jacobian at every step was both about $30$ times more expensive and less accurate, the drift injected by re-estimation outweighing the benefit of a current linearization.

\subsection{Cost and Hyperparameters}

The dominant cost per case is the initial decomposition, $140$ metric-vector products or roughly $280$ partial forward-pass equivalents; the $30$ steering steps add about two more each. Table~\ref{tab:app:hyper} lists the settings, with the GPT-2 Small values for comparison.

\begin{table}[t]
	\centering
	\small
	\begin{tabular}{lcc}
		\toprule
		Setting & GPT-2 Small & 8B models \\
		\midrule
		Fisher truncation $K$ & 5000 & 1000 \\
		Randomized rank $r$ & --- & 50 \\
		Oversampling & --- & 20 \\
		Concept efficiency $c_{\min}$ & 0.3 & 0.5 \\
		Tikhonov $\alpha$ & \multicolumn{2}{c}{median positive eigenvalue} \\
		Finite difference $\varepsilon$ & \multicolumn{2}{c}{$10^{-4}$} \\
		Steps $T$ & \multicolumn{2}{c}{30} \\
		Jacobian & recomputed & frozen at $\mathbf{h}_0$ \\
		\bottomrule
	\end{tabular}
	\caption{Settings of the steering procedure. Entries spanning both columns are shared.}
	\label{tab:app:hyper}
\end{table}

\subsection{Error Sources and a Known Boundary}

Four approximations separate the computed direction from $\delta\mathbf{h}^{\star}$, and each is bounded or checked where possible. The finite-difference product carries a truncation error of order $\varepsilon^{2}$ and is reproducible to machine precision under repetition. The top-$K$ truncation discards tail mass, and the spectrum of $G$ is stable under changes of $K$. The randomized decomposition retains $50$ directions, comfortably above the measured effective rank at every layer. Freezing the Jacobian introduces drift, which was measured to be smaller than the noise of re-estimation. The effect of an inexact metric and covector on the resulting direction is bounded in Theorem~\ref{thm:app:perturb}.

The fourth approximation also marks the boundary of the method. At roughly $75\%$ relative depth the conditioning of $G$ has flattened considerably (\appGeometry), which by the bound of \appUnified{} lowers the gain available from correcting the metric; the ablation gap closes accordingly. This is a property of the geometry at those layers rather than a defect of the approximation, and it is the behavior the analysis of \appDepth{} predicts.
\subsection{Computational Environment}

Experiments ran under Python~3.10 with PyTorch~2.7.1 and Transformers~4.52.4, on NVIDIA RTX~PRO~6000 and H200 accelerators. GPT-2 Small requires $16$~GB of device memory for the exact Jacobian computation; the 8B experiments were run on an $80$~GB accelerator, with $48$~GB the practical minimum. Statistics and figure generation run on CPU.

Total compute is $31.4$ GPU-hours for the $600$ GPT-2 Small cases, $10.8$ for the $450$ Llama-3-8B cases, and $30.8$ for the $900$ Qwen3-8B cases. All randomness is seeded: the synthetic study, the randomized range finder, and the ReFT-r1 initialization each use a fixed seed, and the C4 scan consumes shards and documents in deterministic order. The closed-form solves and the bisection localization are deterministic for a fixed model snapshot, so residual variation comes only from floating-point reductions in GPU kernels.
\section{Experimental Protocol}
\label{app:protocol}

This appendix records the construction of the evaluation set, the calibration of step sizes, and the localization of the comparison points, in the detail needed to reproduce the numbers of Section~5 of the main paper.

\subsection{Concepts and Contexts}

The three concepts are verb-inflection contrasts: third-person singular (\emph{run}$\to$\emph{runs}), progressive (\emph{run}$\to$\emph{running}), and past tense (\emph{run}$\to$\emph{ran}). Each is specified by a set of counterfactual token pairs whose two members differ only in that inflection, following the protocol of \citet{park2026information}; only pairs whose two forms are both single tokens in the model's vocabulary are retained, so that the concept probability of Section~5 is well defined at a single prediction step.

Contexts are natural-language sequences drawn from the validation split of C4 \citep{raffel2020exploring}. A position is accepted when the model's three most probable next tokens all belong to the base set $\{y_i^{0}\}$ or the target set $\{y_i^{1}\}$ and carry a cumulative probability of at least $0.7$, the criterion of \citet{park2026information}; this ensures the concept is genuinely at issue at that position rather than being a marginal alternative. Table~\ref{tab:app:screening} reports the counts. Between $44$ and $166$ thousand sequences are scanned per concept, of which a few hundred meet the criterion and a subset is retained: $50$ contexts per layer on GPT-2 Small and Llama-3-8B, and $100$ on Qwen3-8B. The same contexts are reused across the layers of a given model, so a context contributes one case per layer, which is why the pooled rows of Table~1 are treated as descriptive rather than confirmatory.

\begin{table}[t]
	\centering
	\small
	\begin{tabular}{llrrr}
		\toprule
		Model & Concept & Scanned & Accepted & Used \\
		\midrule
		GPT-2 Small & \textit{third} & 165{,}925 & 500 & 50 \\
		& \textit{ing}   & 139{,}174 & 500 & 50 \\
		& \textit{past}  & 57{,}313  & 500 & 50 \\
		\midrule
		Llama-3-8B  & \textit{third} & 73{,}204  & 250 & 50 \\
		& \textit{ing}   & 58{,}379  & 250 & 50 \\
		& \textit{past}  & 44{,}209  & 250 & 50 \\
		\midrule
		Qwen3-8B    & \textit{third} & 74{,}071  & 300 & 100 \\
		& \textit{ing}   & 59{,}635  & 300 & 100 \\
		& \textit{past}  & 55{,}210  & 300 & 100 \\
		\bottomrule
	\end{tabular}
	\caption{Construction of the evaluation set. ``Used'' is the number of contexts evaluated per layer; each is evaluated at every layer of that model.}
	\label{tab:app:screening}
\end{table}

\subsection{Targets and Step-Size Calibration}

Each method steers each context toward four levels of the concept probability, $P^{W}\in\{0.3,0.5,0.7,0.9\}$, spanning a mild shift to a strong one. Steering runs for $T=30$ steps, and the step size of every method---FishBack, the Euclidean ablation, and each baseline---is calibrated independently by line search so that its trajectory of $P^{W}$ covers the highest target it can reach. Calibrating each method to its own scale prevents the comparison from turning on a shared step size that happens to suit one of them.

\subsection{Localization of the Comparison Point}

The trajectory of $P^{W}$ is discrete, so a method rarely lands exactly on a target. For each target we bisect in activation space between the two adjacent iterates that bracket it, and evaluate the off-target divergence at the localized crossing. Every method is therefore scored at precisely the same on-target level, and the off-target comparison is not contaminated by residual differences in concept change. A method that fails to reach a target within its calibrated range is recorded as a miss for that (case, target) cell and contributes to the reachability figures of \appReach{} rather than to the off-target comparison.

\subsection{Off-Target Divergence}

The off-target divergence is computed on the distribution renormalized over the complement of the concept tokens, as defined in Section~3 of the main paper, and evaluated at the localized crossing point of the previous subsection. The same definition is used for all methods and all three models. Its relation to the quantity appearing in Theorem~3 is given in \appProofs.

\subsection{Aggregation}

A case contributes to a comparison only when both methods reach the target under consideration. Within a case, the four per-target ratios are collapsed by averaging their logarithms and exponentiating, which is their geometric mean, and a case is retained in the folded comparison only if all four targets are reached by both methods. Cells with fewer than six retained pairs are marked underpowered and excluded from confirmatory claims. Statistical treatment of the resulting samples is described in \appStats.

\section{Statistical Procedure and Full Results}
\label{app:stats}

\subsection{Test and Correction}

All off-target comparisons are paired within a case. For a method $m$ and a case $c$ we form the ratio
\begin{equation}
	R_{m,c}=\frac{\mathrm{KL}^{\mathrm{off}}_{m,c}}{\mathrm{KL}^{\mathrm{off}}_{\mathrm{FishBack},c}},
\end{equation}
so that $R>1$ favors FishBack, and test the one-sided hypothesis $\operatorname{median}\log R>0$ with a Wilcoxon signed-rank test on $\log R$.

The transform is not cosmetic. A ratio is confined to $(0,1)$ when FishBack loses and to $(1,\infty)$ when it wins, so under the null of no difference the distribution of $R$ is asymmetric about $1$ and the signed-rank statistic, which assumes symmetry, is biased. Taking logarithms restores the symmetry the test requires; an early version of our analysis applied the test to raw ratios and produced significant cells that did not survive the corrected procedure.

Two levels of aggregation are reported. At the \emph{target} level a pair is one context at one of the four values of $P^{W}$. At the \emph{folded} level the four ratios of a case are collapsed into their geometric mean, and a case enters only if both methods reach all four targets. Multiplicity is controlled by the Benjamini--Hochberg procedure applied within a family, and a family is defined per model and per aggregation level: at the target level the family spans concept $\times$ layer $\times$ target $\times$ comparator, and at the folded level concept $\times$ layer $\times$ comparator. Table~\ref{tab:app:families} lists their sizes. Cells with fewer than six retained pairs are marked underpowered and excluded from confirmatory claims, though they remain in the released tables.

\begin{table}[t]
	\centering
	\small
	\begin{tabular}{llr}
		\toprule
		Model & Level & Tests \\
		\midrule
		GPT-2 Small & target & 288 \\
		& folded & 72 \\
		Llama-3-8B  & target & 216 \\
		& folded & 45 + 9 \\
		Qwen3-8B    & target & 216 \\
		& folded & 45 + 9 \\
		\bottomrule
	\end{tabular}
	\caption{Size of each Benjamini--Hochberg family. On the 8B models the folded level is corrected separately for the baseline comparisons and the geometric ablation.}
	\label{tab:app:families}
\end{table}

\subsection{Outcome by Layer}

Table~\ref{tab:app:sigbylayer} reports, for every layer of every model, how many folded cells were testable and how many survived correction. The pattern is the one the main paper describes. On GPT-2 Small the proportion rises with depth and is complete at the deepest layer. On both 8B models it is high at the early and middle layers and collapses at roughly $75\%$ relative depth, where one cell of eighteen survives on Llama-3-8B and two of eighteen on Qwen3-8B.

\begin{table}[t]
	\centering
	\small
	\begin{tabular}{llcc}
		\toprule
		Model & Layer & Tested & Significant \\
		\midrule
		GPT-2 Small & 3  & 16 & 10 \\
		& 6  & 16 & 14 \\
		& 9  & 18 & 16 \\
		& 11 & 18 & 18 \\
		\midrule
		Llama-3-8B  & 8  & 17 & 15 \\
		& 16 & 18 & 10 \\
		& 24 & 18 & 1 \\
		\midrule
		Qwen3-8B    & 6  & 18 & 15 \\
		& 18 & 18 & 18 \\
		& 27 & 18 & 2 \\
		\bottomrule
	\end{tabular}
	\caption{Folded cells surviving Benjamini--Hochberg correction, counted over the three concepts and all comparators. Underpowered cells are excluded from the tested column.}
	\label{tab:app:sigbylayer}
\end{table}

\subsection{Outcome by Comparison}

Aggregating instead over layers, Table~\ref{tab:app:sigbyfamily} separates the two comparison families. The ablation against Euclidean steering and the comparisons against fixed baselines behave alike, which is consistent with the attribution argument of Section~5.4: what the fixed baselines lose to FishBack is largely what the Euclidean ablation loses to it.

\begin{table}[t]
	\centering
	\small
	\begin{tabular}{llcc}
		\toprule
		Model & Family & Tested & Significant \\
		\midrule
		GPT-2 Small & baselines & 224 & 173 \\
		& ablation  & 48  & 38 \\
		\midrule
		Llama-3-8B  & baselines & 176 & 75 \\
		& ablation  & 36  & 20 \\
		\midrule
		Qwen3-8B    & baselines & 180 & 115 \\
		& ablation  & 36  & 24 \\
		\bottomrule
	\end{tabular}
	\caption{Target-level cells surviving correction, pooled over concepts, layers, and targets, including the deepest layers of the 8B models.}
	\label{tab:app:sigbyfamily}
\end{table}

\subsection{Released Tables}

The complete per-cell record accompanies the submission: model, concept, layer, comparator, target, sample size, median ratio with interquartile range, median log-ratio, uncorrected and corrected $p$-values, and the underpowered and significance flags, for all $900$ cells across both aggregation levels. Every number quoted in the main paper can be traced to a row of that table.

\section{Reachability}
\label{app:reachability}

The off-target comparison of the main paper conditions on both methods reaching a target, so it is silent on how often a target is reached at all. This appendix reports that quantity. A method reaches a target when its calibrated trajectory of $P^{W}$ crosses the prescribed level within the step budget; the denominator is the set of cases with a valid base point, pooled over the three concepts and the four targets.

\begin{table}[t]
	\centering
	\small
	\begin{tabular}{lcccc}
		\toprule
		& \multicolumn{4}{c}{GPT-2 Small} \\
		\cmidrule(lr){2-5}
		Method & L3 & L6 & L9 & L11 \\
		\midrule
		FishBack   & 100.0 & 100.0 & 99.6 & 91.2 \\
		Euclidean  & 100.0 & 100.0 & 100.0 & 100.0 \\
		CAA        & 66.7 & 85.2 & 100.0 & 100.0 \\
		ActAdd     & 12.7 & 14.8 & 73.8 & 96.9 \\
		ITI        & 71.4 & 88.6 & 99.2 & 100.0 \\
		ReFT       & 69.8 & 84.1 & 96.9 & 99.2 \\
		RepSurgery & 60.3 & 86.4 & 98.5 & 100.0 \\
		\bottomrule
	\end{tabular}
	
	\vspace{0.6em}
	
	\begin{tabular}{lccc@{\hspace{1.2em}}ccc}
		\toprule
		& \multicolumn{3}{c}{Llama-3-8B} & \multicolumn{3}{c}{Qwen3-8B} \\
		\cmidrule(lr){2-4} \cmidrule(lr){5-7}
		Method & L8 & L16 & L24 & L6 & L18 & L27 \\
		\midrule
		FishBack   & 98.2 & 100.0 & 99.7 & 75.2 & 67.9 & 68.5 \\
		Euclidean  & 98.2 & 99.8 & 100.0 & 36.3 & 22.7 & 23.3 \\
		CAA        & 56.3 & 62.4 & 87.1 & 55.0 & 51.7 & 54.3 \\
		ActAdd     & 28.6 & 41.9 & 52.7 & 52.2 & 48.6 & 46.0 \\
		ITI        & 54.8 & 39.3 & 55.9 & 45.7 & 44.9 & 40.8 \\
		ReFT       & 67.5 & 65.0 & 71.0 & 48.1 & 45.2 & 50.8 \\
		RepSurgery & 58.7 & 59.8 & 57.0 & 58.1 & 56.5 & 44.2 \\
		\bottomrule
	\end{tabular}
	\caption{Percentage of targets reached, pooled over the three concepts and the four levels of $P^{W}$.}
	\label{tab:app:reach}
\end{table}

Three observations follow from Table~\ref{tab:app:reach}.

The fixed-direction baselines are the constrained ones, and their difficulty is concentrated at the early layers. On GPT-2 Small, CAA reaches two thirds of the targets at layer $3$ and all of them at layer $9$; ActAdd reaches one target in eight at layer $3$. This is why the paired samples of the main comparison shrink at those layers, and why ActAdd in particular is compared on a reduced set of cases: the ratios reported for it describe the cases where it did reach the target, which are the cases where it was competitive.

FishBack and the Euclidean ablation reach targets at similar rates on GPT-2 Small and Llama-3-8B, with one exception in the wrong direction: at GPT-2 Small's layer $11$ the ablation reaches every target while FishBack reaches $91.2\%$. Reachability is not uniformly in our favor, and we report it as measured.

On Qwen3-8B the ordering changes. The Euclidean ablation reaches between $22.7\%$ and $36.3\%$ of targets, well below every fixed baseline, while FishBack reaches between $67.9\%$ and $75.2\%$. The gap has a specific cause rather than a general one. The applied direction of \appApprox{} fixes the concept efficiency at $c_{\min}$, so progress toward the target does not depend on the conditioning of the metric; the ablation removes the geometric component but inherits the same construction, and on this model the resulting direction advances too slowly along the concept to reach the higher targets within the step budget. The reachability gap on Qwen3-8B is therefore evidence that the geometric component contributes to progress on this model, not a separate advantage of the method.

\section{A Metric View of Existing Steering Methods}
\label{app:unified}

The direction of Theorem~2 is optimal for the metric $G$. Existing methods use other directions, and this appendix quantifies what that costs. Throughout, $G\succ0$ on the effective tangent space, the concept constraint is $\mathbf{q}^{\top}\delta=\rho$, and
\begin{equation}
	C_G(\delta)=\tfrac12\,\delta^{\top}G\,\delta
\end{equation}
is the second-order output cost of Theorem~1. We write $\delta_G^{\star}$ for the optimum and $\Delta C_G(\tilde\delta)=C_G(\tilde\delta)-C_G(\delta_G^{\star})$ for the excess cost of any competing direction.

\subsection{Excess Cost Is a Squared Distance}

\begin{theorem}[Fisher--Pythagorean identity]
	\label{thm:app:pythagoras}
	For any $\tilde\delta$ satisfying $\mathbf{q}^{\top}\tilde\delta=\rho$,
	\begin{equation}
		\Delta C_G(\tilde\delta)=\tfrac12\big(\tilde\delta-\delta_G^{\star}\big)^{\top}G\big(\tilde\delta-\delta_G^{\star}\big).
	\end{equation}
\end{theorem}

\begin{proof}
	Let $e=\tilde\delta-\delta_G^{\star}$. Both directions satisfy the constraint, so $\mathbf{q}^{\top}e=0$. Stationarity of $\delta_G^{\star}$ gives $G\delta_G^{\star}=\nu\mathbf{q}$ for some scalar $\nu$, whence
	\begin{equation*}
		e^{\top}G\delta_G^{\star}=\nu\,e^{\top}\mathbf{q}=0 .
	\end{equation*}
	Expanding $C_G(\delta_G^{\star}+e)$ and cancelling the cross term leaves $C_G(\delta_G^{\star})+\tfrac12 e^{\top}Ge$.
\end{proof}

The identity has the form of a Pythagorean relation in the geometry of $G$: the excess cost is the squared $G$-length of the error, and the error is $G$-orthogonal to the optimum. Diagonalizing $G=\sum_i\lambda_i\mathbf{u}_i\mathbf{u}_i^{\top}$ and writing $e_i=\mathbf{u}_i^{\top}e$ gives
\begin{equation}
	\label{eq:app:spectral_excess}
	\Delta C_G(\tilde\delta)=\tfrac12\sum_i\lambda_i e_i^{2},
\end{equation}
so the penalty for a given error depends on where the error points, not only on how large it is.

\begin{corollary}[Errors along flat directions are invisible]
	\label{cor:app:lowcurv}
	If $e$ has no component along the $k$ leading eigendirections of $G$, then
	\begin{equation}
		\Delta C_G(\tilde\delta)\le\tfrac12\,\lambda_{k+1}\lVert e\rVert^{2}.
	\end{equation}
\end{corollary}

\begin{proof}
	Immediate from \eqref{eq:app:spectral_excess}, since the surviving terms carry eigenvalues at most $\lambda_{k+1}$.
\end{proof}

This is why a method can use the wrong direction and still perform acceptably. With effective ranks in the range measured in \appGeometry, most of the spectrum is negligible, and an error confined to that part of the space costs almost nothing at the output. Cosine similarity to the optimal direction is therefore the wrong diagnostic; the relevant quantity is the $G$-weighted spectral error of \eqref{eq:app:spectral_excess}.

\subsection{Cost of Steering Under a Proxy Metric}

A linear steering method that uses the correct concept covector but an incorrect metric $M\succ0$ produces
\begin{equation}
	\delta_M=\frac{\rho\,M^{-1}\mathbf{q}}{\mathbf{q}^{\top}M^{-1}\mathbf{q}} .
\end{equation}

\begin{theorem}[Cost ratio of a proxy metric]
	\label{thm:app:costratio}
	\begin{equation}
		R_G(M;\mathbf{q})=\frac{C_G(\delta_M)}{C_G(\delta_G^{\star})}
		=\frac{\big(\mathbf{q}^{\top}M^{-1}GM^{-1}\mathbf{q}\big)\big(\mathbf{q}^{\top}G^{-1}\mathbf{q}\big)}{\big(\mathbf{q}^{\top}M^{-1}\mathbf{q}\big)^{2}},
	\end{equation}
	and $R_G(M;\mathbf{q})\ge1$ with equality if and only if $M^{-1}\mathbf{q}\parallel G^{-1}\mathbf{q}$.
\end{theorem}
\begin{proof}
	Substituting $\delta_M$ into $C_G$ gives
	\begin{equation*}
		C_G(\delta_M)=\frac{\rho^{2}}{2}\,\frac{\mathbf{q}^{\top}M^{-1}GM^{-1}\mathbf{q}}{\big(\mathbf{q}^{\top}M^{-1}\mathbf{q}\big)^{2}},
	\end{equation*}
	while Theorem~2 gives $C_G(\delta_G^{\star})=\rho^{2}/\big(2\,\mathbf{q}^{\top}G^{-1}\mathbf{q}\big)$; dividing yields the stated expression. For the inequality, set $\mathbf{a}=G^{1/2}M^{-1}\mathbf{q}$ and $\mathbf{b}=G^{-1/2}\mathbf{q}$, so that
	\begin{equation*}
		\begin{aligned}
			\mathbf{a}^{\top}\mathbf{b}&=\mathbf{q}^{\top}M^{-1}\mathbf{q},\\
			\lVert\mathbf{a}\rVert^{2}&=\mathbf{q}^{\top}M^{-1}GM^{-1}\mathbf{q},\\
			\lVert\mathbf{b}\rVert^{2}&=\mathbf{q}^{\top}G^{-1}\mathbf{q},
		\end{aligned}
	\end{equation*}
	and $R_G(M;\mathbf{q})=\lVert\mathbf{a}\rVert^{2}\lVert\mathbf{b}\rVert^{2}/(\mathbf{a}^{\top}\mathbf{b})^{2}\ge1$ by Cauchy--Schwarz. Equality holds exactly when $\mathbf{a}\parallel\mathbf{b}$, that is, when $M^{-1}\mathbf{q}\parallel G^{-1}\mathbf{q}$.
\end{proof}
The ratio takes a more transparent form after whitening. Let
\begin{equation}
	\nu=\frac{G^{-1/2}\mathbf{q}}{\sqrt{\mathbf{q}^{\top}G^{-1}\mathbf{q}}},
	\qquad
	B=G^{1/2}M^{-1}G^{1/2},
\end{equation}
so that $\lVert\nu\rVert=1$ and $B\succ0$ measures the mismatch between the proxy and the true metric.

\begin{theorem}[Cost ratio as a weighted dispersion]
	\label{thm:app:cv}
	With $B=\sum_i b_i\mathbf{r}_i\mathbf{r}_i^{\top}$ and weights $w_i=(\mathbf{r}_i^{\top}\nu)^{2}$ summing to one,
	\begin{equation}
		\label{eq:app:cv}
		R_G(M;\mathbf{q})=\frac{\nu^{\top}B^{2}\nu}{\big(\nu^{\top}B\nu\big)^{2}}
		=1+\frac{\operatorname{Var}_w(b)}{\big(\mathbb{E}_w b\big)^{2}} .
	\end{equation}
\end{theorem}
\begin{proof}
	Write $\mathbf{q}=G^{1/2}\mathbf{z}$, so that $\nu=\mathbf{z}/\lVert\mathbf{z}\rVert$. Each factor of Theorem~\ref{thm:app:costratio} becomes
	\begin{equation*}
		\begin{aligned}
			\mathbf{q}^{\top}M^{-1}\mathbf{q}&=\lVert\mathbf{z}\rVert^{2}\,\nu^{\top}B\nu,\\
			\mathbf{q}^{\top}M^{-1}GM^{-1}\mathbf{q}&=\lVert\mathbf{z}\rVert^{2}\,\nu^{\top}B^{2}\nu,\\
			\mathbf{q}^{\top}G^{-1}\mathbf{q}&=\lVert\mathbf{z}\rVert^{2},
		\end{aligned}
	\end{equation*}
	using $B=G^{1/2}M^{-1}G^{1/2}$ for the first two. Substituting and cancelling $\lVert\mathbf{z}\rVert^{2}$ gives $R_G=\nu^{\top}B^{2}\nu/(\nu^{\top}B\nu)^{2}$. Diagonalizing $B=\sum_i b_i\mathbf{r}_i\mathbf{r}_i^{\top}$ and setting $w_i=(\mathbf{r}_i^{\top}\nu)^{2}$, which sum to one since $\lVert\nu\rVert=1$, yields $\nu^{\top}B\nu=\sum_i w_ib_i$ and $\nu^{\top}B^{2}\nu=\sum_i w_ib_i^{2}$. Hence
	\begin{equation*}
		R_G=\frac{\sum_i w_ib_i^{2}}{\big(\sum_i w_ib_i\big)^{2}}
		=\frac{\mathbb{E}_w\big[b^{2}\big]}{\big(\mathbb{E}_w b\big)^{2}}
		=1+\frac{\operatorname{Var}_w(b)}{\big(\mathbb{E}_w b\big)^{2}},
	\end{equation*}
	the last step by $\operatorname{Var}_w(b)=\mathbb{E}_w[b^{2}]-(\mathbb{E}_w b)^{2}$.
\end{proof}
Equation \eqref{eq:app:cv} says that the penalty for using the wrong metric is the squared coefficient of variation of the mismatch spectrum, weighted by how the concept direction distributes over that spectrum. A proxy that is wrong by a uniform factor costs nothing, because a global rescaling is absorbed by the constraint. A proxy is expensive only when it misscales the directions the concept actually occupies, by differing amounts.

Two consequences follow. \begin{corollary}[Ceiling on the attainable gain]
	\label{cor:app:kantorovich}
	If the spectrum of $B$ lies in $[a,b]$ with $0<a\le b$, then
	\begin{equation}
		\label{eq:app:kantorovich}
		R_G(M;\mathbf{q})\le\frac{(a+b)^{2}}{4ab}.
	\end{equation}
\end{corollary}

\begin{proof}
	By Theorem~\ref{thm:app:cv}, $R_G=\mathbb{E}_w[b^{2}]/(\mathbb{E}_w b)^{2}$ where $w$ is a probability distribution supported on the spectrum of $B$, hence on $[a,b]$. Over all such distributions this ratio is maximized at an extreme point of the moment set, which for a ratio of this form is a two-point distribution on $\{a,b\}$. Writing the mass at $a$ as $t$,
	\begin{equation*}
		\frac{\mathbb{E}_w[b^{2}]}{(\mathbb{E}_w b)^{2}}=\frac{ta^{2}+(1-t)b^{2}}{\big(ta+(1-t)b\big)^{2}},
	\end{equation*}
	whose maximum over $t\in[0,1]$ is attained at $t=b/(a+b)$ and equals $(a+b)^{2}/(4ab)$.
\end{proof}
so the attainable gain from correcting the metric is capped by the conditioning of the mismatch, and a well-conditioned mismatch leaves little to gain. And the gain vanishes when the weights $w$ concentrate on a single eigenvalue of $B$, that is, when the concept direction aligns with one eigendirection of the mismatch rather than spreading across several. These are the two quantities examined empirically in \appDepth.

\subsection{Additive Methods}

Methods that add a fixed direction $\mathbf{v}$ take the form $\mathbf{h}\leftarrow\mathbf{h}+\alpha\mathbf{v}$. Rescaled to meet the same concept constraint, this is
\begin{equation}
	\delta_{\mathbf{v}}(\rho)=\frac{\rho\,\mathbf{v}}{\mathbf{q}^{\top}\mathbf{v}},
	\qquad \mathbf{q}^{\top}\mathbf{v}\neq0,
\end{equation}
which is the proxy-metric form with $M=I$ and the covector replaced by $\mathbf{v}$: CAA takes $\mathbf{v}=\mu_+-\mu_-$, the difference of class means, and ActAdd the difference of a single pair of prompt activations. Their cost follows directly rather than through Theorem~\ref{thm:app:costratio}, whose statement uses the correct covector.

\begin{proposition}[Cost of a fixed direction]
	\label{prop:app:fixeddir}
	For any $\mathbf{v}$ with $\mathbf{q}^{\top}\mathbf{v}\neq0$,
	\begin{equation}
		\frac{C_G(\delta_{\mathbf{v}})}{C_G(\delta_G^{\star})}
		=\frac{\big(\mathbf{v}^{\top}G\mathbf{v}\big)\big(\mathbf{q}^{\top}G^{-1}\mathbf{q}\big)}{\big(\mathbf{q}^{\top}\mathbf{v}\big)^{2}}\;\ge\;1,
	\end{equation}
	with equality if and only if $\mathbf{v}\parallel G^{-1}\mathbf{q}$.
\end{proposition}

\begin{proof}
	Substituting $\delta_{\mathbf{v}}=\rho\mathbf{v}/(\mathbf{q}^{\top}\mathbf{v})$ gives $C_G(\delta_{\mathbf{v}})=\tfrac{\rho^{2}}{2}\,\mathbf{v}^{\top}G\mathbf{v}/(\mathbf{q}^{\top}\mathbf{v})^{2}$, while $C_G(\delta_G^{\star})=\rho^{2}/\big(2\,\mathbf{q}^{\top}G^{-1}\mathbf{q}\big)$; dividing gives the expression. Setting $\mathbf{a}=G^{1/2}\mathbf{v}$ and $\mathbf{b}=G^{-1/2}\mathbf{q}$ makes it $\lVert\mathbf{a}\rVert^{2}\lVert\mathbf{b}\rVert^{2}/(\mathbf{a}^{\top}\mathbf{b})^{2}\ge1$ by Cauchy--Schwarz, with equality exactly when $G^{1/2}\mathbf{v}\parallel G^{-1/2}\mathbf{q}$.
\end{proof}

An empirical direction therefore pays twice: once for the flat metric, and once for departing from $G^{-1}\mathbf{q}$.

We note a limitation of this framing. Any direction that changes the concept can be written as $M^{-1}\hat{\mathbf{q}}$ for some positive definite $M$ and some covector $\hat{\mathbf{q}}$, and the representation is not unique; exhibiting one is therefore not by itself informative. What carries content is the cost, which is why the results above are stated in terms of $R_G$ rather than in terms of the existence of an implicit metric.

\section{Depth Dependence of the Advantage}
\label{app:depth}

\appUnified{} identifies two quantities that govern how much correcting the metric can help: the conditioning of the mismatch between the proxy and the true metric, which caps the attainable gain, and the spread of the concept direction across the mismatch spectrum, which determines how much of that cap is realized. This appendix examines whether the realized gain behaves as those quantities predict, layer by layer.

\subsection{Measured Quantities}

The theoretical quantities of \appUnified{} are defined through $B=G^{1/2}M^{-1}G^{1/2}$, which requires a square root and an inverse of the metric at every case. We instead record two computable proxies for the same underlying property, namely how the concept covector sits relative to the spectrum of $G$:
\begin{equation}
	\mathbf{q}^{\top}G\,\mathbf{q},
	\qquad
	\frac{\lVert H^{1/2}J\mathbf{q}\rVert}{\lVert J\mathbf{q}\rVert},
\end{equation}
the first measuring the curvature the concept direction encounters, the second the fraction of the pulled-back concept direction that the output Fisher retains. Neither is the dispersion of \eqref{eq:app:cv}; both increase when the concept direction occupies directions the metric treats as expensive, which is the regime in which correcting the metric matters.

The realized gain of a case is the log-ratio of off-target divergences between the Euclidean ablation and FishBack, averaged over the targets both reach. Correlations are Spearman rank correlations over cases, pooled across the three concepts within each layer.

\subsection{Results}

\begin{table}[t]
	\centering
	\small
	\begin{tabular}{llcrr}
		\toprule
		Model & Layer & $n$ & $\rho_{\,\mathbf{q}^{\top}G\mathbf{q}}$ & $\rho_{\,\text{align}}$ \\
		\midrule
		GPT-2 Small & 3  & 13  & $-0.16$ & $-0.17$ \\
		& 6  & 22  & $0.26$  & $0.16$ \\
		& 9  & 27  & $0.26$  & $-0.30$ \\
		& 11 & 28  & $0.07$  & $-0.41^{*}$ \\
		\midrule
		Llama-3-8B  & 8  & 122 & $0.08$  & $-0.03$ \\
		& 16 & 117 & $0.24^{*}$ & $0.21^{*}$ \\
		& 24 & 93  & $0.42^{***}$ & $0.50^{***}$ \\
		\midrule
		Qwen3-8B    & 6  & 106 & $-0.13$ & $-0.16$ \\
		& 18 & 67  & $0.23$  & $0.26^{*}$ \\
		& 27 & 68  & $0.56^{***}$ & $0.62^{***}$ \\
		\bottomrule
	\end{tabular}
	\caption{Spearman correlations between the two spectral proxies and the per-case log gain over Euclidean steering. Significance is marked at $p<0.05$ ($^{*}$) and $p<0.001$ ($^{***}$).}
	\label{tab:app:depth}
\end{table}

On both 8B models the pattern is the same and it is monotone in depth. At the early layers neither proxy predicts the gain: the correlations are near zero and not significant, which is what one expects when the geometry is rich enough that the correction helps almost everywhere, so there is little variation across cases for a predictor to explain. At the middle layers weak but significant correlations appear. At the deepest layers both proxies become strong predictors, reaching $0.42$ and $0.50$ on Llama-3-8B and $0.56$ and $0.62$ on Qwen3-8B, all significant after correction.

Read together with \appStats{} and \appGeometry{}, this gives a coherent account of the deepest layers. The effective rank of $G$ there falls to single digits, so the subspace in which the correction can act is narrow; the number of cells surviving correction collapses to one of eighteen on Llama-3-8B and two of eighteen on Qwen3-8B; and what gain remains is concentrated in precisely the cases where the concept direction still occupies expensive directions of the metric. The advantage does not decay uniformly, as a general failure of the method would produce. It narrows to the cases the geometry still supports, which is the behavior \appUnified{} describes.

\subsection{GPT-2 Small}

GPT-2 Small does not reproduce this pattern, and we report it as measured. Its correlations remain weak at every layer, and at the deepest layer the alignment proxy is significantly \emph{negatively} correlated with the gain. Two circumstances distinguish this model. Its condition numbers remain of order $10^{7}$ at every layer, varying by less than a factor of two across depth, so the flattening that drives the pattern above does not occur here; consistently with that, the advantage over Euclidean steering persists at every GPT-2 Small layer, and the proportion of cells surviving correction is highest at layer $11$ (\appStats). The correlations are also estimated on between $13$ and $28$ cases, against roughly a hundred on the 8B models; at that sample size a correlation of $0.3$ is not distinguishable from zero.

We therefore read the negative coefficient at layer $11$ as unexplained rather than as evidence against the account, and note that the mechanism it would speak to concerns a regime that this model does not enter.

\section{The Output-Layer Metric at Intermediate Layers}
\label{app:ladder}

\citet{park2026information} establish the geometry of the softmax head: the Fisher information $H$ is the metric on the final representation, and steering in its dual coordinates changes a target concept while preserving unrelated semantics. Their analysis is stated at the output, where it is exact. Steering in practice, however, is applied several layers upstream, and the metric that governs an intervention there is not $H$ but its pullback $G=J^{\top}HJ$. This appendix measures what that difference is worth.

The relationship between the two is a strict specialization. At $\ell=L$ the map $f$ is the identity, so $J=I$, $\mathbf{q}=\boldsymbol{\beta}_W$, and Theorem~2 reduces to
\begin{equation}
	\delta\boldsymbol{\lambda}^{\star}\propto H^{-1}\boldsymbol{\beta}_W ,
\end{equation}
which is the dual-coordinate direction of \citet{park2026information}. Their construction is thus the boundary case of ours in which no transport is required. The question this appendix answers is what happens to that construction when the intervention moves upstream and the transport is omitted.

\subsection{Setup}

Three directions are compared at every layer, differing only in the metric used to raise the concept covector:
\begin{equation}
	M=I,
	\qquad
	M=H,
	\qquad
	M=J^{\top}HJ .
\end{equation}
The middle rung applies the output-layer metric of \citet{park2026information} at the intervention layer, unchanged. All three use the same pulled-back covector $\mathbf{q}=J^{\top}\boldsymbol{\beta}_W$, the same fixed concept efficiency, the same adaptive regularization, the same step-size calibration, and the same bisection to the exact crossing, so runs differ only in the metric. For $M=H$ the regularized solve uses the Woodbury identity on the low-rank factorization of the truncated Fisher, which is exact and needs no randomized decomposition; the output-layer metric is therefore solved more accurately than the pullback is at 8B scale, and the comparison is conservative in its favor.

\subsection{Results}

Table~\ref{tab:app:ladder_main} reports the paired comparison on the cohort used throughout the paper, in which both methods reach all four targets on the same case. On GPT-2 Small and Qwen3-8B the pullback produces less off-target distortion at every layer, with median ratios from $1.35$ to $3.56$, all significant after correction. On Llama-3-8B this cohort is too small to support inference, for a reason given below.

\begin{table}[t]
	\centering
	\small
	\begin{tabular}{llccc}
		\toprule
		Model & Layer & Ratio & Win\% & $n$ \\
		\midrule
		GPT-2 Small & 3  & 1.35$^{*}$ & 67 & 108 \\
		& 6  & 1.62$^{*}$ & 86 & 121 \\
		& 9  & 2.72$^{*}$ & 89 & 109 \\
		& 11 & 2.60$^{*}$ & 97 & 86 \\
		\midrule
		Llama-3-8B  & 8  & 0.98 & 44 & 18 \\
		& 16 & 1.16 & 80 & 5 \\
		& 24 & 1.33 & 60 & 5 \\
		\midrule
		Qwen3-8B    & 6  & 1.41$^{*}$ & 82 & 22 \\
		& 18 & 1.98$^{*}$ & 83 & 12 \\
		& 27 & 3.56$^{*}$ & 100 & 13 \\
		\bottomrule
	\end{tabular}
	\caption{Off-target divergence under the output-layer metric of \citet{park2026information}, applied at the intervention layer, relative to the pullback metric, at matched concept change; values above $1$ favor the pullback. Cohort: both methods reach all four targets. $^{*}$ significant after Benjamini--Hochberg correction.}
	\label{tab:app:ladder_main}
\end{table}

The all-four requirement is severe here, because the output-layer metric reaches targets far less often than the pullback once it is used upstream: $92$--$96\%$ of targets on GPT-2 Small, but only $22$--$37\%$ on Llama-3-8B and $13$--$20\%$ on Qwen3-8B, against $98$--$100\%$ for the pullback. At a per-target rate near $0.2$ the probability of reaching all four falls below one percent, which is why the 8B cohorts hold a few dozen cases rather than a few hundred. Table~\ref{tab:app:ladder_wide} therefore repeats the comparison on the wider cohort requiring only the same single target.

\begin{table}[t]
	\centering
	\small
	\begin{tabular}{llcccc}
		\toprule
		Model & Layer & Ratio & Win\% & $n$ & Cells \\
		\midrule
		Llama-3-8B & 8  & 1.11 & 64 & 146 & 2/12 \\
		& 16 & 2.60 & 73 & 135 & 1/12 \\
		& 24 & 2.98 & 71 & 81  & 0/12 \\
		\midrule
		Qwen3-8B   & 6  & 1.52 & 75 & 169 & 4/12 \\
		& 18 & 1.91 & 77 & 105 & 6/12 \\
		& 27 & 2.45 & 86 & 138 & 5/12 \\
		\bottomrule
	\end{tabular}
	\caption{The same comparison on the single-target cohort. ``Cells'' counts concept--target cells significant after correction out of twelve.}
	\label{tab:app:ladder_wide}
\end{table}

Two features of the tables should be stated explicitly. The ratio grows toward the output layer on all three models---$1.35$ to $2.60$ across GPT-2 Small, $1.11$ to $2.98$ across Llama-3-8B, $1.41$ to $3.56$ across Qwen3-8B---so the transport matters most where least of it is applied; we report this as measured and offer no mechanism for it. And the three pairwise comparisons rest on different cohorts, since each requires both of its own methods to succeed, so the ratios do not compose and a comparison against the Euclidean ablation cannot be inferred from the other two.

\subsection{Why the Untransported Metric Underperforms}

The direction cosines in Table~\ref{tab:app:ladder_cos} identify the mechanism, and show that it differs by scale. On both 8B models the direction produced by the output-layer metric is nearly parallel to the raw covector, with median cosine near $0.95$ at the early and middle layers: applied upstream, $H$ barely redirects anything, so the resulting direction is close to the flat one, and its off-target divergence is statistically indistinguishable from the Euclidean ablation's on both models. On GPT-2 Small the cosine is lower, between $0.32$ and $0.84$, so $H$ does redirect---and the redirection is unhelpful, the Euclidean ablation achieving lower off-target divergence at three of four layers.

\begin{table}[t]
	\centering
	\small
	\begin{tabular}{llccc}
		\toprule
		Model & Layer & $\cos(\delta_H,\delta_G)$ & $\cos(\mathbf{q},\delta_H)$ & $\cos(\mathbf{q},\delta_G)$ \\
		\midrule
		GPT-2 Small & 3  & 0.11 & 0.84 & 0.13 \\
		& 6  & 0.14 & 0.81 & 0.19 \\
		& 9  & 0.20 & 0.72 & 0.33 \\
		& 11 & 0.29 & 0.32 & 0.45 \\
		\midrule
		Llama-3-8B  & 8  & 0.67 & 0.95 & 0.71 \\
		& 16 & 0.72 & 0.95 & 0.75 \\
		& 24 & 0.73 & 0.92 & 0.82 \\
		\midrule
		Qwen3-8B    & 6  & 0.65 & 0.95 & 0.68 \\
		& 18 & 0.72 & 0.95 & 0.76 \\
		& 27 & 0.78 & 0.90 & 0.87 \\
		\bottomrule
	\end{tabular}
	\caption{Median cosines between the raw solved directions, before the concept-efficiency decomposition. $\delta_H$ uses the output-layer metric, $\delta_G$ the pullback.}
	\label{tab:app:ladder_cos}
\end{table}

Either way the conclusion is the same. The geometry that \citet{park2026information} identify is the right geometry, but it is the right geometry \emph{for the final representation}. Carried upstream unchanged, it either collapses toward the flat metric it was meant to replace, as on the 8B models, or redirects in a way that does not serve the objective, as on GPT-2 Small. The Jacobian is what makes it applicable where steering is actually performed.

One case deserves separate mention. At Qwen3-8B's deepest layer the Euclidean ablation gap has closed, as reported in \appDepth: the pullback and the flat metric are no longer distinguishable there. The untransported metric, however, remains clearly worse at that layer, with a median ratio of $3.56$ and every case favoring the pullback. What disappears at depth is the gain available from correcting the geometry, not the cost of correcting it in the wrong space.

\subsection{Summary}

Across three models and ten layers, steering under the transported metric produces less off-target distortion than steering under the output-layer metric applied upstream, in every cell measured, reaching significance on GPT-2 Small and Qwen3-8B; the evidence on Llama-3-8B is directionally consistent but does not reach significance in either cohort. Since the untransported metric offers no reliable improvement over flat steering on either 8B model, and is worse than flat steering on GPT-2 Small, the gap that this paper closes is not the identification of the softmax geometry, which \citet{park2026information} supply, but its transport to the layer at which interventions are made.


\end{document}